\documentclass{article}

\usepackage{microtype}
\usepackage{graphicx}
\usepackage{booktabs} 

\usepackage{hyperref}
\usepackage[a-2u]{pdfx}

\usepackage[accepted]{icml2021}

\usepackage{stmaryrd}

\usepackage{xcolor}
\definecolor{our-green}{rgb}{0.56, 0.692, 0.195}
\definecolor{our-darkgreen}{rgb}{0.297, 0.348, 0.105}
\definecolor{our-yellow}{rgb}{0.881, 0.611, 0.142}
\definecolor{ourbrown}{rgb}{0.772,0.432,0.102}
\definecolor{our-red}{rgb}{0.923,0.386,0.209}
\definecolor{our-blue}{rgb}{0.368,0.507,0.71}

\hypersetup{
	colorlinks,
	linkcolor=our-darkgreen,
	urlcolor=our-darkgreen,
	citecolor=our-red}

\usepackage{graphicx}
\graphicspath{{graphics/}}

\usepackage{caption,subcaption}
\usepackage{makecell}
\renewcommand{\arraystretch}{1.2}

\usepackage{xspace}
\makeatletter
\DeclareRobustCommand\onedot{\futurelet\@let@token\@onedot}
\def\@onedot{\ifx\@let@token.\else.\null\fi\xspace}

\newcommand\eg{{e.g}\onedot}  
\newcommand\ie{{i.e}\onedot}

\newcommand\wrt{w.r.t\onedot} 

\makeatother

\newcommand\Figref[1]{Fig.~\ref{#1}}
\newcommand\Tabref[1]{Tab.~\ref{#1}}
\newcommand\Secref[1]{Sec.~\ref{#1}}
\newcommand\Algref[1]{Module~\ref{#1}}

\newcommand{\method}{CombOptNet\xspace}
\newcommand{\knapsack}{\textsc{Knapsack}\xspace}
\newcommand{\wsc}{\textsc{WSC}\xspace}
\newcommand{\gt}{ground-truth\xspace}
\newcommand{\rc}{\textsc{RC}\xspace}
\newcommand{\bbgm}{BB-GM\xspace}
\newcommand{\suppl}{Supplementary material\xspace}
\newcommand{\bin}{\textit{binary}\xspace}
\newcommand{\den}{\textit{dense}\xspace}

\usepackage{amsthm}

\newtheorem{prop}{Proposition}
\theoremstyle{definition}
\newtheorem*{example}{Example}

\usepackage{amsmath,amsfonts,bm}

\def\va{{\bm{a}}}
\def\vb{{\bm{b}}}
\def\vc{{\bm{c}}}

\def\ve{{\bm{e}}}

\def\vo{{\bm{o}}}

\def\vx{{\bm{x}}}
\def\vy{{\bm{y}}}

\def\mA{{\bm{A}}}

\DeclareMathAlphabet{\mathsfit}{\encodingdefault}{\sfdefault}{m}{sl}
\SetMathAlphabet{\mathsfit}{bold}{\encodingdefault}{\sfdefault}{bx}{n}

\def\sN{{\mathbb{N}}}

\def\sZ{{\mathbb{Z}}}

\newcommand{\R}{\mathbb{R}}

\DeclareMathOperator{\sign}{sign}

\newcommand\proj[1]{#1_\text{proj}}
\renewcommand{\d}{\mathrm{d}}

\DeclareMathOperator\dist{\mathrm{dist}}
\newcommand\Tstrut{\rule{0pt}{2.4ex}}

\usepackage[noend]{algpseudocode}
\algrenewcommand\algorithmicindent{1em}
\algrenewcommand{\algorithmiccomment}[1]{%
\bgroup\hskip2em\textcolor{our-darkgreen}{//~\textsl{#1}}\egroup}
\algrenewcommand{\Return}{\State\textbf{return}\ }
\algnewcommand{\Save}{\State\textbf{save}\ }
\algnewcommand{\Load}{\State\textbf{load}\ }

\usepackage[
	textsize=tiny,
	textwidth=2.5cm,
	colorinlistoftodos,
	backgroundcolor=our-green!30!white,
	linecolor=our-green!50!white
	]{todonotes}

\usepackage{xr}
\makeatletter
\newcommand*{\addFileDependency}[1]{
  \typeout{(#1)}
  \@addtofilelist{#1}
  \IfFileExists{#1}{}{\typeout{No file #1.}}
}
\makeatother

\usepackage{flushend}
\usepackage{multirow}
\usepackage{adjustbox}
\include{dcolumn}
\newcolumntype{d}{D{+}{\,$\pm$\,}{3,3}}
\newcolumntype{e}[1]{D{+}{\,$\pm$\,}{#1}}

\newcommand{\headimg}[1]{\includegraphics[width=.8\linewidth]{#1}}

\icmltitlerunning{CombOptNet: Fit the Right NP-Hard Problem by Learning Integer Programming Constraints}

\begin{document}
\twocolumn[
\icmltitle{CombOptNet: Fit the Right NP-Hard Problem by \\Learning Integer Programming Constraints}




\begin{icmlauthorlist}
\icmlauthor{Anselm Paulus}{mpi}
\icmlauthor{Michal Rol\'inek}{mpi}
\icmlauthor{V\'\i t Musil}{muni}
\icmlauthor{Brandon Amos}{fb}
\icmlauthor{Georg Martius}{mpi}
\end{icmlauthorlist}

\icmlaffiliation{mpi}{Max-Planck-Institute for Intelligent Systems, T\"ubingen, Germany}
\icmlaffiliation{muni}{Masaryk University, Brno, Czechia}
\icmlaffiliation{fb}{Facebook AI Research, USA}

\icmlcorrespondingauthor{Anselm Paulus \& Georg Martius}{firstname.lastname@tuebingen.mpg.de}

\icmlkeywords{}

\vskip 0.3in

\begin{center}
  \headimg{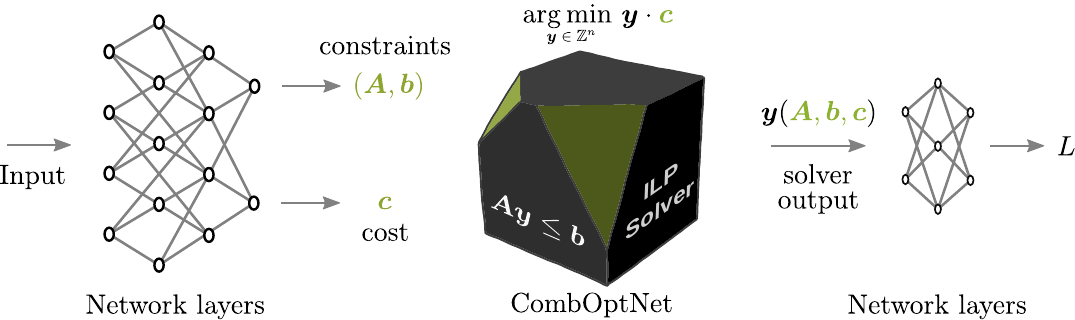}
  \captionof{figure}{\method{} as a module in a deep architecture.}
  \label{fig:overview}
\end{center}
\vskip 0.2in

]



\printAffiliationsAndNotice{}  

\begin{abstract}
Bridging logical and algorithmic reasoning with modern machine learning
techniques is a fundamental challenge with potentially transformative impact.
On the algorithmic side, many \textsc{NP-hard} problems can be expressed as integer
programs, in which the constraints play the role of their ``combinatorial
specification.'' In this work, we aim to integrate integer programming solvers
into neural network architectures as layers capable of learning \emph{both}
the cost terms and the constraints. The resulting end-to-end trainable
architectures jointly extract features from raw data and
solve a suitable (learned) combinatorial problem with state-of-the-art
integer programming solvers.
We demonstrate the potential of such layers with an extensive performance analysis on synthetic data and with a demonstration on a competitive computer vision keypoint matching benchmark.
\end{abstract}

\section{Introduction}

It is becoming increasingly clear that to advance artificial
intelligence, we need to dramatically enhance the reasoning, algorithmic,
logical, and symbolic capabilities of data-driven models.
Only then we can aspire to match humans in their astonishing ability to perform complicated
abstract tasks such as playing chess only based on visual input.
While there are decades worth of research directed at solving complicated abstract tasks
\emph{from their abstract formulation}, it seems very difficult to align these
methods with deep learning architectures needed for processing raw inputs. Deep
learning methods often struggle to implicitly acquire the abstract reasoning
capabilities to solve and generalize to new tasks.  Recent work has
investigated more structured paradigms that have more explicit reasoning
components, such as layers capable of convex optimization.  In this paper, we
focus on combinatorial optimization, which has been well-studied and captures
nontrivial reasoning capabilities over discrete objects. Enabling its
unrestrained usage in machine learning models should fundamentally enrich the
set of available components.

On the technical level, the main challenge of incorporating combinatorial
optimization into the model typically amounts to \emph{non-differentiability}
of methods that operate with discrete inputs or outputs. Three basic approaches
to overcome this are to a) develop ``soft'' continuous versions of the discrete
algorithms \citep{wang2019satnet, zanfir2018deep};
b) adjust the topology of neural network architectures to express certain algorithmic behaviour
\citep{graves2014neural, graves2016hybrid, battaglia2018}; c) provide an
informative gradient approximation for the discrete algorithm
\citep{vlastelica2019differentiation, berthet2020learning}.  While the last
strategy requires nontrivial theoretical considerations, it can resolve the
non-differentiability in the strongest possible sense; without any compromise
on the performance of the original discrete algorithm.
We follow this approach.

The most successful generic approach to combinatorial optimization is integer linear programming (ILP). Integrating ILPs as building blocks of differentiable models is
challenging because of the nontrivial dependency of the solution on the
cost terms and on the constraints.
Learning parametrized cost terms has been addressed in
\citet{vlastelica2019differentiation, berthet2020learning, ferber2020mipaal},
the learnability of constraints is, however, unexplored. At the same time, the
constraints of an ILP are of critical interest due to their \textbf{remarkable
expressive power}. Only by modifying the constraints, one can formulate a number
of diverse combinatorial problems (\textsc{shortest-path}, \textsc{matching},
\textsc{max-cut}, \knapsack, \textsc{travelling salesman}). In that
sense, learning ILP constraints corresponds to \textbf{learning the
combinatorial nature} of the problem at hand.

In this paper, we propose a backward pass (gradient computation) for ILPs covering their \textbf{full specification},
allowing to use blackbox ILPs as combinatorial layers at any point
in the architecture.
This layer can jointly learn the cost terms and the
constraints of the integer program, and as such it aspires to achieve
\textbf{universal combinatorial expressivity}.
We demonstrate the potential of this method on multiple
tasks. First, we extensively analyze the performance on synthetic data.
This includes the inverse optimization task of
recovering an unknown set of constraints, 
and a \knapsack problem specified in plain text descriptions.
Finally, we demonstrate the applicability to real-world tasks
on a competitive computer vision keypoint matching benchmark.

\subsection{Related Work}\label{sec:related_work}
\paragraph{Learning for combinatorial optimization.}
Learning methods can powerfully augment classical combinatorial optimization
methods with data-driven knowledge.
This includes work that learns how to solve
combinatorial optimization problems to improve upon traditional
solvers that are otherwise computationally expensive or intractable, \eg
by using reinforcement learning
\citep{zhang2000solving,bello2016neural,khalil2017learning,nazari2018reinforcement},
learning graph-based algorithms
\citep{velivckovic2017graph,velivckovic2019neural,wilder2019end},
learning to branch \citep{balcan2018learning},
solving SMT formulas \citep{balunovic2018learning}
and TSP instances \citep{kool2018attention}.
\citet{nair2020solving} have recently scaled up learned MIP
solvers on non-trivial production datasets.
In a more general computational paradigm,
\citet{graves2014neural, graves2016hybrid} parameterize
and learn Turing machines.

\paragraph{Optimization-based modeling for learning.}
In the other direction, optimization serves as a useful modeling
paradigm to improve the applicability of machine learning models and to
add domain-specific structures and priors.
In the continuous setting, differentiating through optimization
problems is a foundational topic as it enables optimization algorithms
to be used as a layer in end-to-end trainable models
\citep{domke2012generic,gould2016differentiating}.
This approach has been recently studied in the convex setting in
OptNet \citep{amos2017optnet} for quadratic programs,
and more general cone programs in
\citet[Section~7.3]{amos2019differentiable}
and
\citet{agrawal2019differentiable,agrawal2019differentiating}.
One use of this paradigm is to incorporate the knowledge of a downstream
optimization-based task into a predictive model
\citep{elmachtoub2020smart,donti2017task}.
Extending beyond the convex setting, optimization-based modeling
and differentiable optimization are used for
sparse structured inference \citep{niculae2018sparsemap},
\textsc{MaxSat} \citep{wang2019satnet},
submodular optimization \citep{djolonga2017differentiable}
mixed integer programming \citep{ferber2020mipaal},
and discrete and combinational settings
\citep{vlastelica2019differentiation,berthet2020learning}.
Applications of optimization-based modeling include
computer vision \citep{rolinek2020deep, rolinek2020cvpr},
reinforcement learning \citep{dalal2018safe,amos2019cem,vlastelica2020neuroalgorithmic},
game theory \citep{ling2018game}, and
inverse optimization
\citep{tan2020learning},
and meta-learning \citep{bertinetto2018meta,lee2019meta}.

\section{Problem description}\label{sec:background}

Our goal is to incorporate an ILP as a differentiable layer in neural networks
that inputs both \textbf{constraints and objective} coefficients and outputs
the corresponding ILP solution.

Furthermore, we aim to embed ILPs in a \textbf{blackbox manner}: 
On the forward pass, we run the unmodified optimized solver, making no compromise on its performance.
The task is to propose an \textbf{informative gradient} for the solver as it is.
We never modify, relax, or soften the solver.

We assume the following form of a bounded integer program:
\begin{equation} \label{E:BILP}
	\min_{\vy\in Y}\, \vc \cdot \vy
		\qquad \text{subject to} \qquad
	\mA\vy \le \vb,
\end{equation}
where $Y$ is a bounded subset of $\sZ^n$, $n\in\sN$, $\vc\in\R^n$ is the cost vector,
$\vy$ are the variables, $\mA=[\va_1,\ldots,\va_m]\in\R^{m\times n}$ is the matrix of
constraint coefficients and $\vb\in\R^m$ is the bias term.
The point at which the minimum is attained is denoted by $\vy(\mA,\vb,\vc)$.

The task at hand is to provide gradients for the mapping
$(\mA,\vb,\vc)\to\vy(\mA,\vb,\vc)$, in which the triple $(\mA,\vb,\vc)$ is the
specification of the ILP solver containing both the cost and the constraints,
and $\vy(\mA,\vb,\vc) \in Y$ is the optimal solution of the instance.

\begin{example} \label{ex:knapsack}
The ILP formulation of the \knapsack problem can be written as
\begin{equation} \label{E:BILP-knapsack}
	\max_{\vy\in \{0, 1\}^n}  \vc\cdot \vy
		\qquad \text{subject to} \qquad
	\va\cdot \vy \le b,
\end{equation}
where $\vc=[c_1,\ldots,c_n]\in\R^n$ are the prices of the items,
$\va=[a_1,\ldots,a_n]\in\R^n$ their weights and $b\in\R$ the knapsack capacity.
\end{example}

Similar encodings can be found for many more - often \textsc{NP-hard} -
combinatorial optimization problems including those mentioned in the
introduction. Despite the apparent difficulty of solving ILPs, modern highly
optimized solvers \cite{gurobi, cplex2009v12} can routinely find optimal
solutions to instances with thousands of variables.

\subsection{The main difficulty.}

\paragraph{Differentiability.}

Since there are finitely many available values of $\vy$, the mapping $(\mA,\vb,\vc)\to\vy(\mA,\vb,\vc)$ is piecewise constant; and as such, its \textbf{true gradient is zero} almost everywhere.
Indeed, a small perturbation of the constraints or of the cost does \emph{typically} not cause a change in the optimal ILP solution.
The zero gradient has to be suitably supplemented.

Gradient surrogates w.r.t.~objective coefficients $\vc$ have been studied intensively 
\citep[see \eg][]{elmachtoub2020smart, vlastelica2019differentiation, ferber2020mipaal}.
Here, we focus on the differentiation w.r.t.~constraints coefficients $(\mA,\vb)$
that has been unexplored by prior works.

\paragraph{LP vs.\ ILP: Active constraints.}

In the LP case, the integrality constraint on $Y$ is removed.
As a result, in the typical case, the optimal solution can be written as the unique solution to a linear system determined by the set of active constraints.
This captures the relationship between the constraint matrix and the optimal solution. Of course, this relationship is differentiable.

However, in the case of an ILP the \textbf{concept of active constraints vanishes}.
There can be optimal solutions for which no constraint is tight.
Providing gradients for nonactive-but-relevant constraints is the principal difficulty.
The complexity of the interaction between the constraint set and the optimal solution is reflecting the \textsc{NP-Hard} nature of ILPs and is the reason why relying on the LP case is of little help.

\section{Method}


First, we reformulate the gradient problem as a descend direction task. We have to resolve
an issue that the suggested gradient update $\vy-\d\vy$ to the optimal solution $\vy$ is typically unattainable, \ie $\vy-\d\vy$ is not a feasible integer point.
Next, we generalize the concept of active constraints.
We substitute the binary information ``active/nonactive'' by a continuous proxy based on Euclidean distance.




\paragraph{Descent direction.}

On the backward pass, the gradient of the layers following the ILP solver is given.
Our aim is to propose a direction of change to the constraints and to the cost such that the solution of the updated ILP moves towards the negated incoming gradient's direction (\ie the descent direction).

Denoting a loss by $L$, let $\mA$, $\vb$, $\vc$ and the incoming gradient $\d\vy=\partial L/\partial\vy$ at the point $\vy=\vy(\mA,\vb,\vc)$ be given.
We are asked to return a gradient corresponding to $\partial L/\partial\mA$, $\partial L/\partial\vb$ and $\partial L/\partial\vc$.
Our goal is to find directions $\d\mA$, $\d\vb$ and $\d\vc$ for which the distance between the updated solution $\vy(\mA-\d\mA,\vb-\d\vb,\vc-\d\vc)$ and the target $\vy-\d\vy$ decreases the most.

If the mapping $\vy$ is differentiable, it leads to the correct gradients $\partial L/\partial\mA=\partial L/\partial \vy\cdot \partial \vy/\partial\mA$ (analogously for $\vb$ and $\vc$).
See Proposition~\ref{P:y-diff} in the \suppl, for the precise formulation and for the proof.
The main advantage of this formulation is that it is \textbf{meaningful even in the discrete case}.

However, every ILP solution $\vy(\mA-\d\mA,\vb-\d\vb,\vc-\d\vc)$ is restricted to integer points and its ability to approach the point $\vy-\d\vy$ is limited unless
$\d\vy$ is also an integer point.
To achieve this, let us decompose
%
%
\begin{equation} \label{E:decomposition}
	\d\vy = \sum_{k=1}^n \lambda_k \Delta_k,
\end{equation}
where $\Delta_k\in\{-1,0,1\}^n$ are some integer points and $\lambda_k\ge 0$ are scalars.
The choice of basis $\Delta_k$ is discussed in a separate paragraph, for now it suffices to know that every point $\vy'_k=\vy-\Delta_k$ is an integer point neighbour of $\vy$ pointing in a ``direction of $-\d\vy$''.
We then address separate problems with $\d\vy$ replaced by the integer updates $\Delta_k$.

In other words, our goal here is to find an update on $\mA$, $\vb$, $\vc$ that eventually pushes the solution closer to $\vy-\Delta_k$.
Staying true to linearity of the standard gradient mapping, we then aim to compose the final gradient as a linear combination of the gradients coming from the subproblems.

\paragraph{Constraints update.}

To get a meaningful update for a realizable change $\Delta_k$, we take a gradient of a piecewise affine local \textit{mismatch} function $P_{\vy'_k}$. 
The definition of $P_{\vy'_k}$ is based on a geometric understanding of the underlying structure.
To that end, we rely on the Euclidean distance between a point and a hyperplane. Indeed, for any point $\vy$ and a given hyperplane, parametrized by vector $\va$ and scalar $b$ as $\vx\mapsto\va\cdot\vx-b$, we have:
\begin{equation}
	\dist(\va,b;\vy)=|\va\cdot\vy-b|/\|\va\|.
\end{equation}
Now, we distinguish the cases based on whether $\vy'_k$ is feasible, \ie $\mA\vy'_k\le\vb$, or not.
The infeasibility of $\vy'_k$ can be caused by one or more constraints.
We then define
\begin{equation} \label{E:proxy-constraints}
	P_{\vy'_k}(\mA,\vb) =
		\begin{cases}
			\min_{j} \dist(\va_j,b_j;\vy)
		        \\ \hspace{2em}
		        \text{\color{our-darkgreen}if $\vy'_k$ is feasible and $\vy'_k\ne\vy$}
				\\
			\sum_{j} \llbracket\va_j\cdot\vy'_k > b_j\rrbracket \dist(\va_j,b_j;\vy'_k)
			    \\ \hspace{2em}
			    \text{\color{our-darkgreen}if $\vy'_k$ is infeasible}
				\\
			\makebox[2em][l]{0}
			 \text{\color{our-darkgreen} if $\vy'_k=\vy$ or $\vy'_k\notin Y$},
		\end{cases}
		\hspace{-1em}
\end{equation}
where $\llbracket\cdot \rrbracket$ is the Iverson bracket.
The geometric intuition behind the suggested mismatch function is described in \Figref{fig:proxy} and its caption.
Note that tighter constraints contribute more to $P_{\vy'_k}$.
In this sense, the mismatch function \textbf{generalizes the concept of active constraints}.
In practice, the minimum is softened to allow multiple constraints to be updated simultaneously. For details, see the \suppl.

\begin{figure}[tb]
	\centering
	\begin{subfigure}[t]{.48\linewidth}
		\centering
		\includegraphics[width=\linewidth]{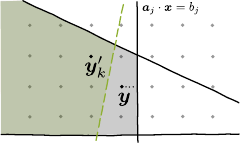}
		\caption{$\vy'_k$ is feasible but $\vy'_k\ne\vy$.}
		\label{fig:proxy-feasible}
	\end{subfigure}
\hskip 0pt plus 1fill
	\begin{subfigure}[t]{.48\linewidth}
		\centering
		\includegraphics[width=\linewidth]{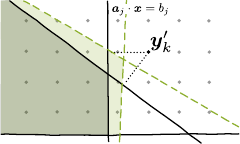}
		\caption{$\vy'_k$ is infeasible.}
		\label{fig:proxy-infeasible}
	\end{subfigure}
	\caption{Geometric interpretation of the suggested constraint update.
		(a) 
			All constraints are satisfied for $\vy'_k$.
			The proxy minimizes the distance to the nearest (``most active'') constraint to make $\vy$ ``less feasible''.
			A possible updated feasible region is shown in green.
		(b) 
			The suggested $\vy'_k$ satisfies one of three constraints.
			The proxy minimizes the distance to violated constraints to make $\vy'_k$ ``more feasible''.
	}
  \label{fig:proxy}
\end{figure}

Imposing linearity and using decomposition~\eqref{E:decomposition}, we define the outcoming gradient~$\d\mA$ as
\begin{equation} \label{E:A-grad}
	\d\mA = \sum_{k=1}^n \lambda_k \frac{\partial P_{\vy'_k}}{\partial\mA}(\mA,\vb)
\end{equation}
and analogously for $\d\vb$, by differentiating with respect~to~$\vb$.
The computation is summarized in \Algref{alg:forward-backward}.

\begin{algorithm}[tb]
\makeatletter
\renewcommand{\ALG@name}{Module}
\makeatother
	\begin{algorithmic}
		\Function{ForwardPass}{$\mA,\vb,\vc$}
		\State $\vy :=$ \textbf{Solver}($\mA,\vb,\vc$)
		\Save $\vy$ and $\mA,\vb,\vc$ for backward pass
		\Return $\vy$
		\EndFunction

		\Function{BackwardPass}{$\d\vy$}
		\Load $\vy$ and $\mA,\vb,\vc$ from forward pass
		\State Decompose $\d\vy = \sum_k \lambda_k \Delta_k$
			\\ \Comment{set $\Delta_k$ as in~\eqref{E:basis} and $\lambda_k$ as in Proposition~\ref{P:basis}}
		\State Calculate the gradients
			\\ \qquad
			$\d\mA^k := \frac{\partial P_{\vy'_k}}{\partial\mA}$,
			$\d\vb^k := \frac{\partial P_{\vy'_k}}{\partial\vb}$,
			$\d\vc^k := \frac{\partial P_{\vy'_k}}{\partial\vc}$
			\\ \Comment{$P_{\vy'_k}$ defined in~\eqref{E:proxy-constraints} and~\eqref{E:proxy-cost}}
		\State Compose
			$\d\mA,\d\vb,\d\vc := \sum_k \lambda_k \left(
					\d\mA^k, \d\vb^k, \d\vc^k
				 \right)$
			\\ \Comment{According to~\eqref{E:A-grad}}
		\Return $\d\mA, \d\vb, \d\vc$
		\EndFunction
	\end{algorithmic}
	\caption{\Tstrut \method}
	\label{alg:forward-backward}
\end{algorithm}

Note that our mapping $\d\vy\mapsto\d\mA,\d\vb$ is homogeneous. It is due to the fact that the whole situation is rescaled to one case (choice of basis) where the gradient is computed and then rescaled back (scalars $\lambda_k$).
The most natural scale agrees with the situation when the ``targets'' $\vy'_k$ are the closest integer neighbors.
This ensures that the situation does not collapse to a trivial solution (zero gradient) and, simultaneously, that we do not interfere with very distant values~of~$\vy$.

This basis selection plays a role of a ``homogenizing hyperparamter'' ($\lambda$ in \cite{vlastelica2019differentiation} or $\varepsilon$ in \cite{berthet2020learning}).
In our case, we explicitly construct a correct basis and do not need to optimize any additional hyperparameter.

\paragraph{Cost update.}

Putting aside distinguishing of feasible and infeasible~$\vy'_k$, the cost update problem has been addressed in multiple previous works.
We use a simple approach of setting the mismatch function such that the resulting update favours $\vy'_k$ over $\vy$ in the updated optimization problem, \ie
\begin{equation} \label{E:proxy-cost}
	P_{\vy'_k}(\vc)
		= \begin{cases}
				\vc\cdot(\vy'_k - \vy)
					\quad\text{\color{our-darkgreen}if $\vy'_k$ is feasible}
					\\
				0
					\hskip2.9em\text{\color{our-darkgreen}if $\vy'_k$ is infeasible or $\vy'_k\notin Y$.}
			\end{cases}
\end{equation}
The gradient $\d\vc$ is then composed analogously as in \eqref{E:A-grad}.

\paragraph{The choice of the basis.}

Denote by $k_1,\ldots,k_n$ the indices of the coordinates in the absolute values of $\d\vy$ in decreasing order, \ie
\begin{equation}
	|\d\vy_{k_1}|
		\ge |\d\vy_{k_2}|
		\ge \cdots
		\ge |\d\vy_{k_n}|
\end{equation}
and set
\begin{equation} \label{E:basis}
	\Delta_k = \sum_{j=1}^k \sign(\d\vy_{k_j}) \ve_{k_j},
\end{equation}
where $\ve_{k}$ is the $k$-th canonical vector.
Therefore, $\Delta_k$ is the (signed) indicator vector of the \emph{first $k$ dominant directions}.

Denote by $\ell$ the largest index for which $|\d\vy_\ell|>0$.
Then the first $\ell$ vectors $\Delta_k$'s are linearly independent and they form a basis of the corresponding subspace.
Therefore, there exist scalars $\lambda_k$'s satisfying decomposition~\eqref{E:decomposition}.

\begin{prop} \label{P:basis}
If
$\lambda_j=|\d\vy_{k_j}|-|\d\vy_{k_{j+1}}|$
for $j=1,\ldots,n-1$ and $\lambda_n=|\d\vy_{k_n}|$,
then representation~\eqref{E:decomposition} holds with
$\Delta_k$'s as in~\eqref{E:basis}.
\end{prop}

An example of a decomposition is shown in \Figref{fig:decomposition}.
Further discussion about the choice of basis and various comparisons can be found in the \suppl.
\begin{figure}[H]
\centering
\includegraphics[width=0.6\linewidth]{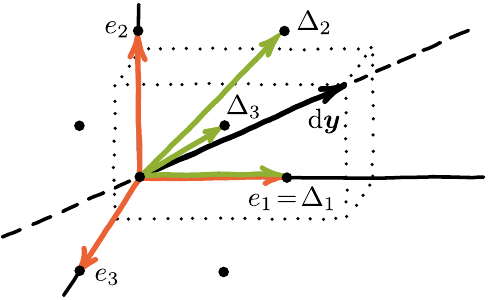}
\caption{All basis vectors $\Delta_k$ (green) point more ``towards the $\d\vy$ direction'' compared to the canonical ones (orange).}
\label{fig:decomposition}
\end{figure}





\paragraph{Constraint parametrization.}

For learning constraints, we have to specify their parametrization.
The representation is of great importance, as it determines how the constraints respond to incoming gradients.
Additionally, it affects the meaning of \textit{constraint distance} by changing the parameter space.

We represent each constraint $(\va_k,b_k)$ as a hyperplane described by its \emph{normal vector} $\va_k$, \emph{distance from the origin} $r_k$ and \emph{offset} $o_k$ of the origin in the global coordinate system as displayed in~\Figref{fig:parametrization-general}.
Consequently $b_k=r_k-\va_k\cdot o_k$.

Compared to the plain parametrization which represents the constraints as a matrix $\mA$ and a vector $\vb$, our slightly overparametrized choice allows the constraints to rotate without requiring to traverse large distance in parameter space (consider \eg a $180^\circ$ rotation).
An illustration is displayed in~\Figref{fig:parametrization-update}.
Comparison of our choice of parametrization to other encodings and its effect on the performance can be found in the~\suppl.

\begin{figure}[tb]
	\centering
	\begin{subfigure}[t]{.48\linewidth}
		\centering
		\includegraphics[width=\linewidth]{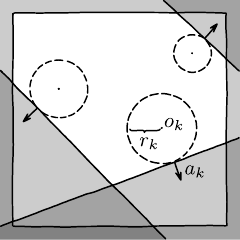}
		\caption{Constraint representation}
		\label{fig:parametrization-general}
	\end{subfigure}
\hskip 0pt plus 1fill
	\begin{subfigure}[t]{.48\linewidth}
		\centering
		\includegraphics[width=\linewidth]{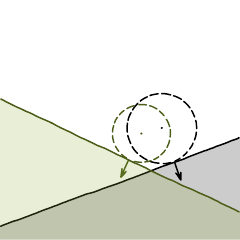}
		\caption{Possible constraint update}
		\label{fig:parametrization-update}
	\end{subfigure}
	\caption{(a)~Each constraint $(\va_k,b_k)$ is parametrized by its normal vector $\va_k$ and a distance $r_k$ to its own origin~$o_k$. (b)~Such a representation allows for easy rotations around the learnable offset~$o_k$ instead of rotating around the static global~origin.}
  \label{fig:parametrization}
\end{figure}

\section{Demonstration \& Analysis}

We demonstrate the potential and flexibility of our method on four tasks.

Starting with an extensive performance analysis on synthetic data, we first
demonstrate the ability to learn multiple constraints simultaneously. For this, we learn a
\emph{static set} of randomly initialized constraints from solved instances,
while using access to the \emph{ground-truth} cost vector $\vc$.

Additionally, we show that the performance of our method on the synthetic datasets also translates to real classes of ILPs. For this we consider a similarly structured task as before, but use the NP-complete \wsc problem to generate the dataset.

Next, we showcase the ability to simultaneously learn the full ILP specification.
For this, we learn a single \emph{input-dependent} constraint
and the cost vector \emph{jointly} from the ground truth solutions of \textsc{knapsack} instances. These instances are encoded as sentence embeddings of their description in natural language.

Finally, we demonstrate that our method is also applicable to real-world problems.
On the task of keypoint matching, we show that our method achieves results that are comparable
to state-of-the-art architectures employing dedicated solvers.
In this example, we \emph{jointly} learn a \emph{static set} of constraints and the cost
vector from \gt matchings.

In all demonstrations, we use \textsc{Gurobi}~\citep{gurobi} to solve the ILPs
during training and evaluation.
Implementation details, a runtime analysis and additional results, such as ablations, other loss functions and more metrics, are provided in the \suppl.
Additionally, a qualitative analysis of the results for the Knapsack demonstration is included.


\subsection{Random Constraints} \label{sec:synthetic}

\paragraph{Problem formulation.}
The task is to learn the constraints $(\mA,\vb)$ corresponding to a fixed ILP. The network has only access to the cost vectors $\vc$ and the \gt ILP solutions $\vy^*$.
Note that the set of constraints perfectly explaining the data does not need to be unique.

\paragraph{Dataset.}
We generate 10 datasets for each cardinality $m=1,2,4,8$ of the \gt constraint set while keeping the dimensionality of the ILP fixed to $n=16$.
Each dataset fixes a set of (randomly chosen) constraints $(\mA,\vb)$ specifying the \gt feasible region of an ILP solver.
For the constraints $(\mA,\vb)$ we then randomly sample cost vectors~$\vc$ and compute the corresponding ILP solution~$\vy^*$~(\Figref{fig:dataset-synthetic}).
\begin{figure}[h]
	\includegraphics[width=\linewidth]{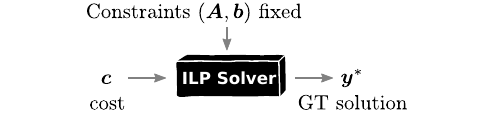}
	\caption{Dataset generation for the \rc demonstration.
	}
	\label{fig:dataset-synthetic}
\end{figure}

The dataset consists of 1\,600 pairs~$(\vc,\vy^*)$ for training
and 1\,000 for testing.
The solution space $Y$ is either constrained to $[-5, 5]^n$ (\emph{dense}) or $[0, 1]^n$ (\emph{binary}).
During dataset generation, we performed a suitable rescaling to ensure a sufficiently large set of feasible solutions.


\paragraph{Architecture.}

The network learns the constraints $(\mA,\vb)$ that specify the ILP solver from \gt pairs $(\vc, \vy^*)$.
Given $\vc$, predicted solution $\vy$ is compared to $\vy^*$ via the MSE~loss and the gradient is backpropagated to the learnable constraints using~\method (\Figref{fig:arch-synthetic}).
\begin{figure}[h]
	\includegraphics[width=\linewidth]{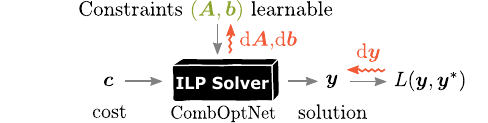}
	\caption{Architecture design for the \rc demonstration.
	}
	\label{fig:arch-synthetic}
\end{figure}

The number of learned constraints matches the number of constraints used for the dataset generation.
Note that if the ILP has no feasible solution, the \method layer output is undefined and any loss or evaluation metric depending on the solution $\vy$ is meaningless.
In practise, updates~\eqref{E:proxy-constraints} push the constraints outwards from the true solution $\vy^*$ leading to a quick emergence of a feasible region.

\paragraph{Baselines.}
We compare \method to three baselines.
Agnostic to any constraints, a simple MLP baseline directly predicts the solution from the input cost vector as the integer-rounded output of a neural network.
The CVXPY baseline uses an architecture similar to ours, only the \Algref{alg:forward-backward} of \method is replaced with the CVXPY implementation \cite{diamond2016CVXPY} of an LP solver
that provides a backward pass proposed by \citet{agrawal2019differentiable}. 
Similar to our method, it receives constraints and a cost vector and outputs the solution of the LP solver greedily rounded to a feasible integer solution.
Finally, we report the performance of always producing the solution of the problem only constrained to the outer region $\vy\in Y$.
This baseline does not involve any training and is purely determined by the dataset.

\paragraph{Results.}
The results are reported in \Figref{fig:random-results}.
\begin{figure}[tb]
  \centering
  \begin{subfigure}[b]{.42\textwidth}
		\centering
		\includegraphics[width=\linewidth]{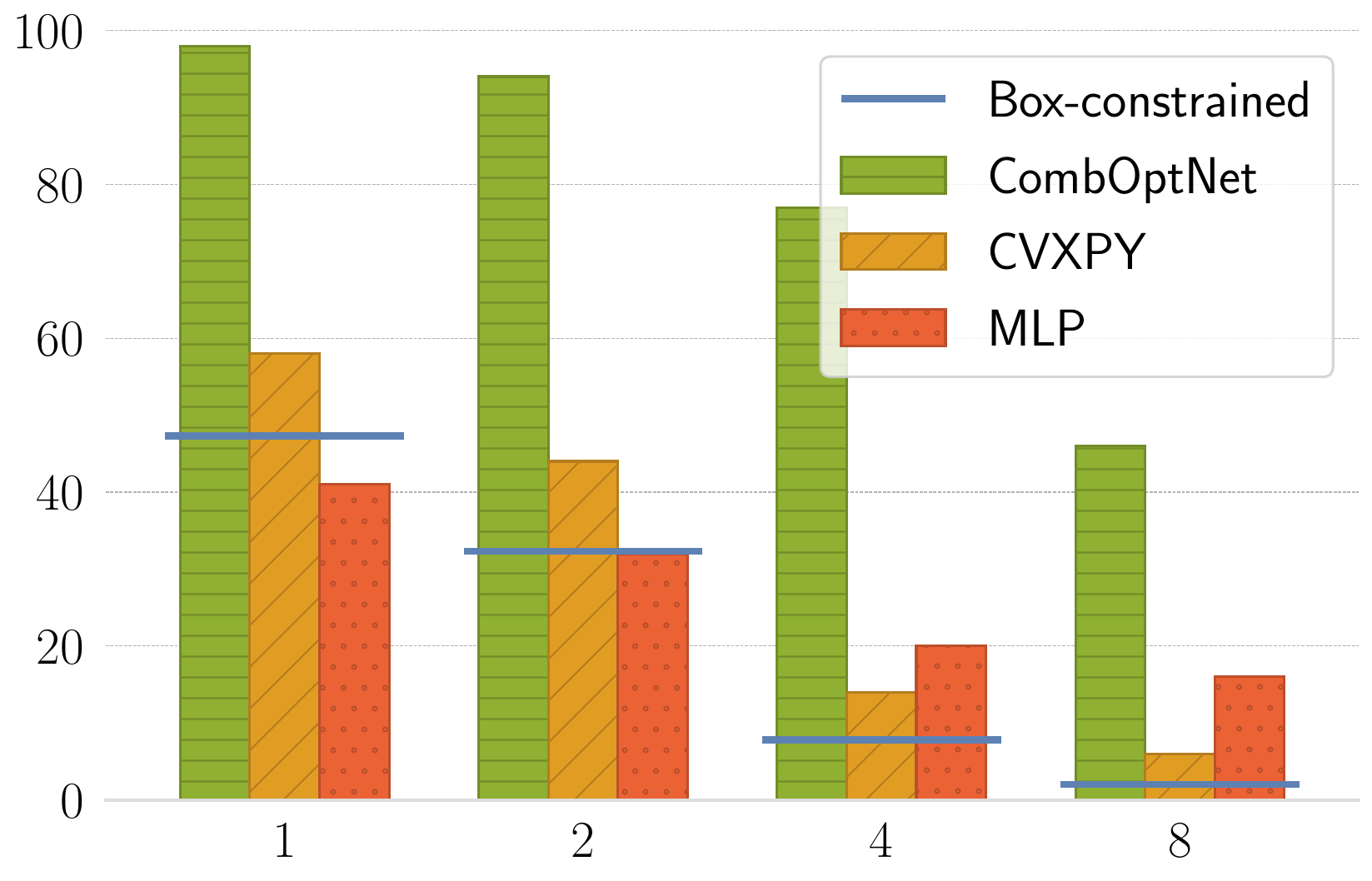}
		\caption{Results on the \textit{binary} datasets.}
	\end{subfigure}
	\hfill
  \begin{subfigure}[b]{.42\textwidth}
		\centering
		\includegraphics[width=\linewidth]{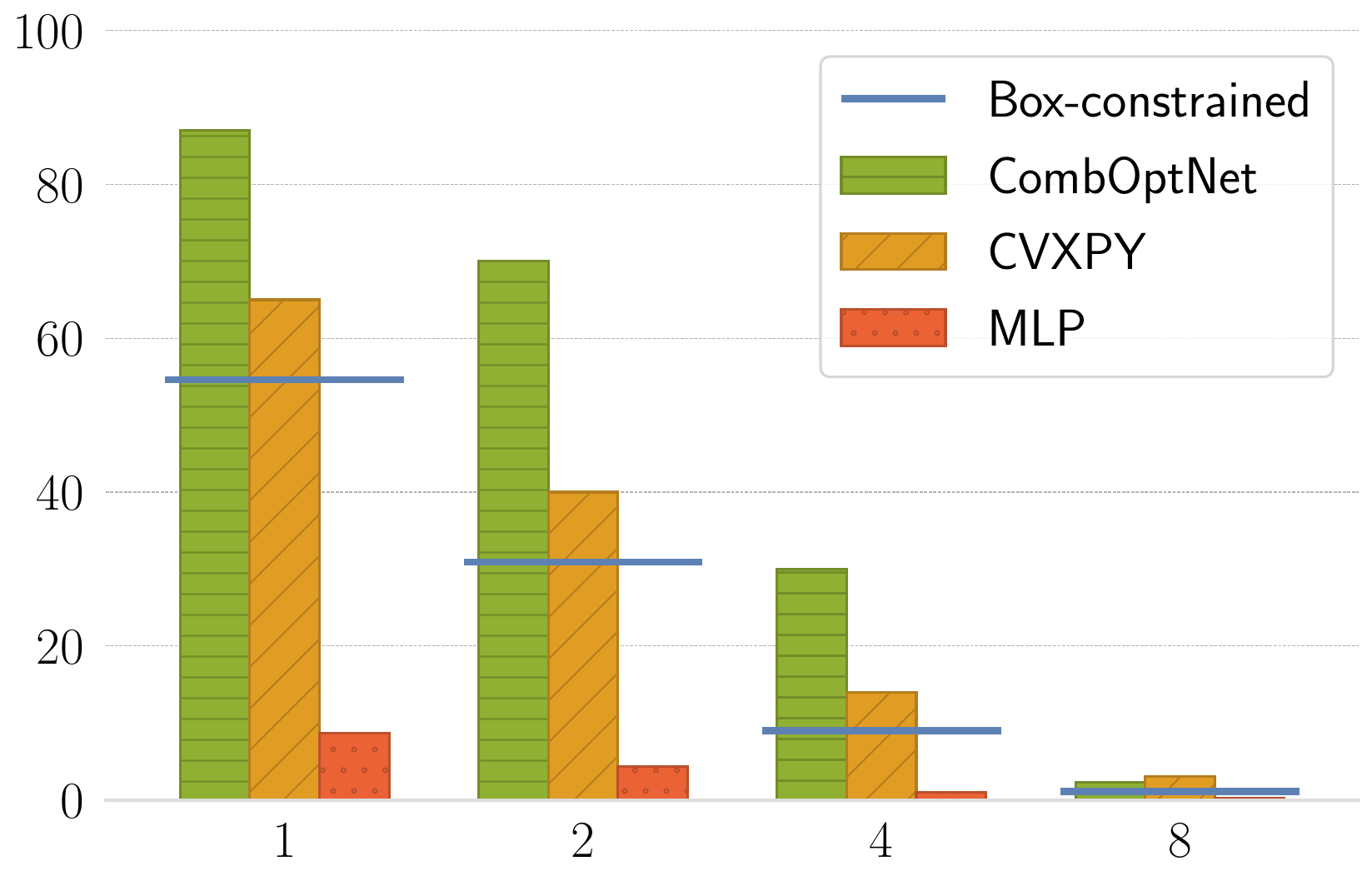}
		\caption{Results on the \textit{dense} datasets.}
	\end{subfigure}
	\caption{Results for the Random Constraints demonstration.
	We report mean accuracy ($\vy=\vy^*$ in \%) over 10 datasets for $1,2,4$ and $8$ ground truth constraints in $16$ dimensions. By Box-constrained we denote the performance of always producing the solution of the problem only constrained to the outer region $\vy\in Y$, which does not involve any training and is purely determined by the dataset.
	}
\label{fig:random-results}
\end{figure}
In the \textit{binary}, case we demonstrate a high accuracy of perfectly predicting the correct solution.
The CVXPY baseline is not capable of matching this, as it is not able to find a set of constraints for the LP problem that mimics the effect of running an ILP solver.
For most cost vectors, CVXPY often predicts the same solution as the unconstrained one and its ability to use constraints to improve is marginal.
The reason is that the LP relaxation of the ground truth problem is far from tight and thus the LP solver proposes many fractional solutions, which are likely to be rounded incorrectly.
This highlights the increased expressivity of the ILP formulation compared to the LP formulation.


Even though all methods decrease in performance in the \textit{dense} case as the number of possible solutions is increased, the trend from the \textit{binary} case continues.
With the increased density of the solution space, the LP relaxation becomes more similar to the ground truth ILP and hence the gap between \method and the CVXPY baseline decreases.

We conclude that \method is especially useful, when the underlying problem is truly difficult (\ie hard to approximate by an LP).
This is not surprising, as \method introduces structural priors into the network that are designed for hard combinatorial problems.

\subsection{Weighted Set Covering} \label{sec:weightedsetcovering}

We show that our performance on the synthetic datasets also translates to traditional classes of ILPs.
Considering a similarly structured architecture as in the previous section, we generate the dataset by solving instances of the NP-complete \wsc problem.

\paragraph{Problem formulation.}

A family $\mathcal C$ of subsets of a universe $U$ is called a covering of $U$ if $\bigcup\mathcal C=U$.
Given $U=\{1,\ldots,m\}$, its covering $\mathcal C=\{S_1,\ldots,S_n\}$ and cost $\vc\colon\mathcal C\to\mathbb R$,
the task is to find the sub-covering $\mathcal C'\subset\mathcal C$ with the lowest total cost $\sum_{S\in\mathcal C'} \vc(S)$.

The ILP formulation of this problem consists of $m$ constraints in $n$ dimensions.
Namely, if $\vy\in\{0,1\}^n$ denotes an indicator vector of the sets in $\mathcal C$,
$a_{kj}=\llbracket k\in S_j\rrbracket$ and $b_k=1$ for $k=1,\ldots m$,
then the specification reads as
\begin{equation}
	\min_{\vy\in Y}\, \sum_j \vc(S_j) \vy_j
		\qquad \text{subject to} \qquad
	\mA\vy \ge \vb.
\end{equation}

\paragraph{Dataset.}

We randomly draw $n$~subsets from the $m$-element universe to form a covering $\mathcal C$.
To increase the variance of solutions, we only allow subsets with no more than 3~elements.
As for the Random Constraints demonstration, the dataset consists of 1\,600 pairs $(\vc,\vy^*)$ for training and 1\,000 for testing.
Here, $\vc$ is uniformly sampled positive cost vector and $\vy^*$ denotes the corresponding optimal solution (\Figref{fig:dataset-wsc}).
We generate 10 datasets for each universe size $m = 4, 5, 6, 7, 8$ with $n=2m$ subsets.
\begin{figure}[h]
	\includegraphics[width=\linewidth]{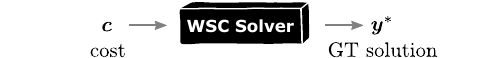}
	\caption{Dataset generation for the \wsc demonstration.}
	\label{fig:dataset-wsc}
\end{figure}

\paragraph{Architecture and Baselines.}
We use the same architecture and compare to the same baselines as in the Random Constraints demonstration~(\Secref{sec:synthetic}).

\paragraph{Results.}

The results are reported in \Figref{fig:wsc-results}.
Our method is still able to predict the correct solution with high accuracy. 
Compared to the previous demonstration, the performance of the LP relaxation deteriorates.
Contrary to the Random Constraints datasets, the solution to the Weighted Set Covering problem never matches the solution of the unconstrained problem, which takes no subset.
This prevents the LP relaxation from exploiting these simple solutions and ultimately leads to a performance drop.
On the other hand, the MLP baseline benefits from the enforced positivity of the cost vector, which leads to an overall reduced number of different solutions in the dataset. 

\begin{figure}[tb]
  \centering
	\centering
	\includegraphics[width=\linewidth]{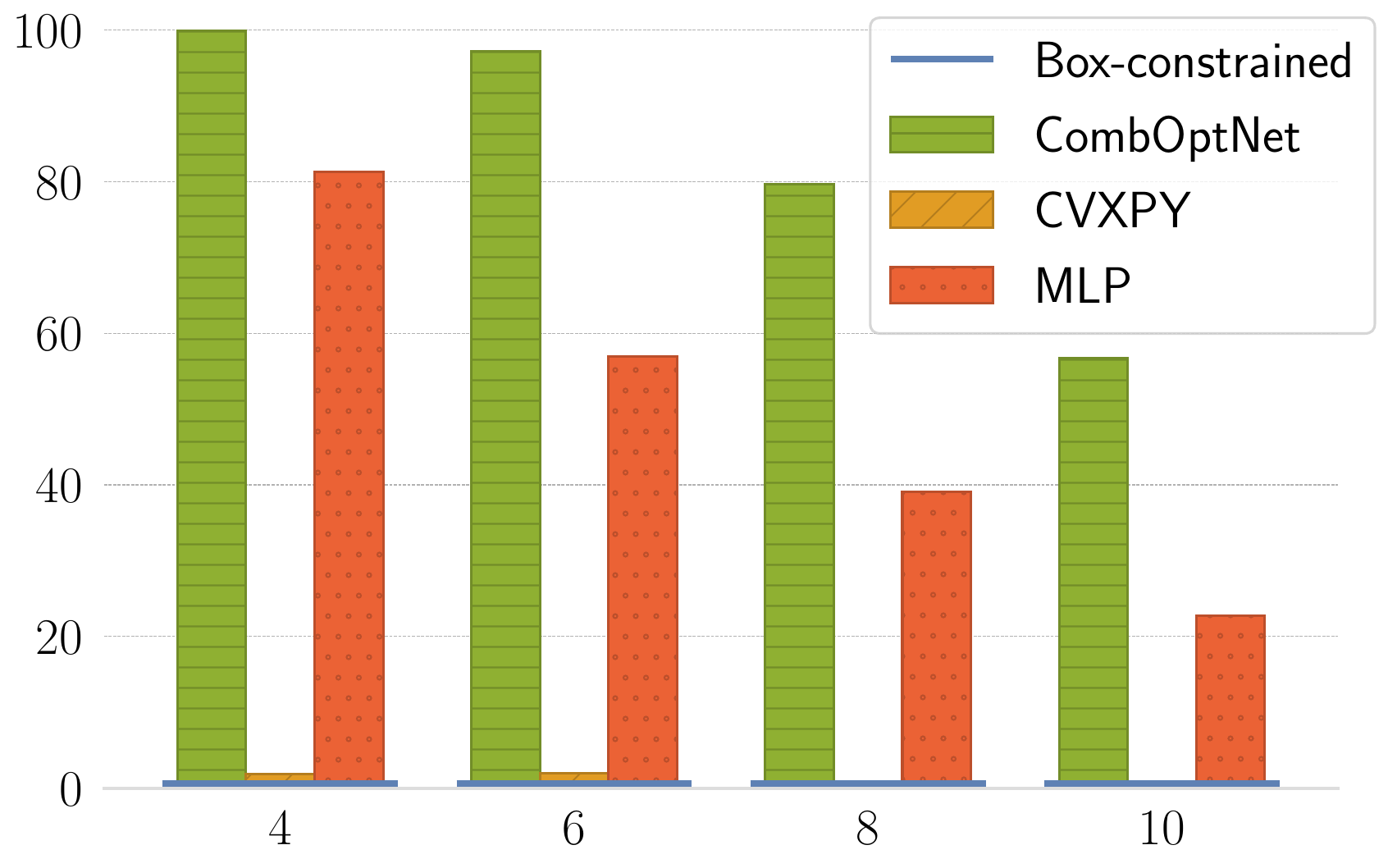}
	\caption{Results of the \wsc demon\-stra\-tion.
	We report mean accuracy ($\vy=\vy^*$ in \%) over 10 datasets for universe sizes $m=4, 6, 8, 10$ and $2m$ subsets.}
\label{fig:wsc-results}
\end{figure}

\subsection{\knapsack from Sentence Description} \label{sec:knapsack}



\paragraph{Problem formulation.}
The task is inspired by a vintage text-based PC game called ``The Knapsack Problem'' \citep{knapsack-game} in which a collection of 10 items is presented to a player including their prices and weights.
The player's goal is to maximize the total price of selected items without exceeding the fixed 100-pound capacity of their knapsack.
The aim is to solve instances of the NP-Hard \knapsack problem~\eqref{ex:knapsack}, from their word descriptions.
Here, the cost $\vc$ and the constraint $(\va,b)$ are learned simultaneously.

\paragraph{Dataset.}
Similarly to the game, a \knapsack instance consists of 10 sentences, each describing one item.
The sentences are preprocessed via the sentence embedding \cite{conneau2017} and the 10 resulting 4\,096-dimensional vectors $\vx$ constitute the input of the dataset.
We rely on the ability of natural language embedding models to capture numerical values, as the other words in the sentence are uncorrelated with them (see an analysis of \citet{wallace2019nlp}).
The indicator vector $\vy^*$ of the optimal solution (\ie item selection) to a knapsack instance is its corresponding label (\Figref{fig:dataset-knapsack}).
The dataset contains 4\,500 training and 500 test pairs $(\vx,\vy^*)$.
\begin{figure}[h]
	\centering
	\includegraphics[width=\linewidth]{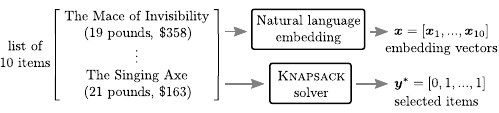}
	\caption{Dataset generation for the \knapsack problem.}
	\label{fig:dataset-knapsack}
\end{figure}


\paragraph{Architecture.}
We simultaneously extract the learnable constraint coefficients~$(\va,\vb)$ and the cost vector~$\vc$ via an MLP from the embedding vectors (\Figref{fig:arch-knapsack}).
\begin{figure}[h]
	\centering
	\includegraphics[width=\linewidth]{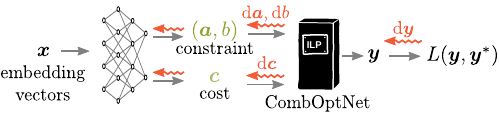}
	\caption{Architecture design for the \knapsack problem.}
	\label{fig:arch-knapsack}
\end{figure}

As only a single learnable constraint is used, which by definition defines a \knapsack problem, the interpretation of this demonstration is a bit different from the other demonstrations.
Instead of learning the type of combinatorial problem, we learn which exact \knapsack problem in terms of item-weights and knapsack capacity needs to be solved.

\paragraph{Baselines.}
We compare to the same baselines as in the Random Constraints demonstration~(\Secref{sec:synthetic}).

\paragraph{Results.}
The results are presented in \Figref{fig:knapsack-results}. While \method is able to predict the correct items for the \knapsack with good accuracy, the baselines are unable to match this.
Additionally, we evaluate the LP relaxation on the ground truth weights and prices, providing an upper bound for results achievable by any method relying on an LP relaxation. The weak performance of this evaluation underlines the NP-Hardness of \knapsack. The ability to embed and differentiate through a dedicated ILP solver leads to surpassing this threshold even when learning from imperfect raw inputs.

\begin{figure}[tb]
	\centering
	\begin{subfigure}[t]{.43\textwidth}
		\centering
		\includegraphics[width=\linewidth]{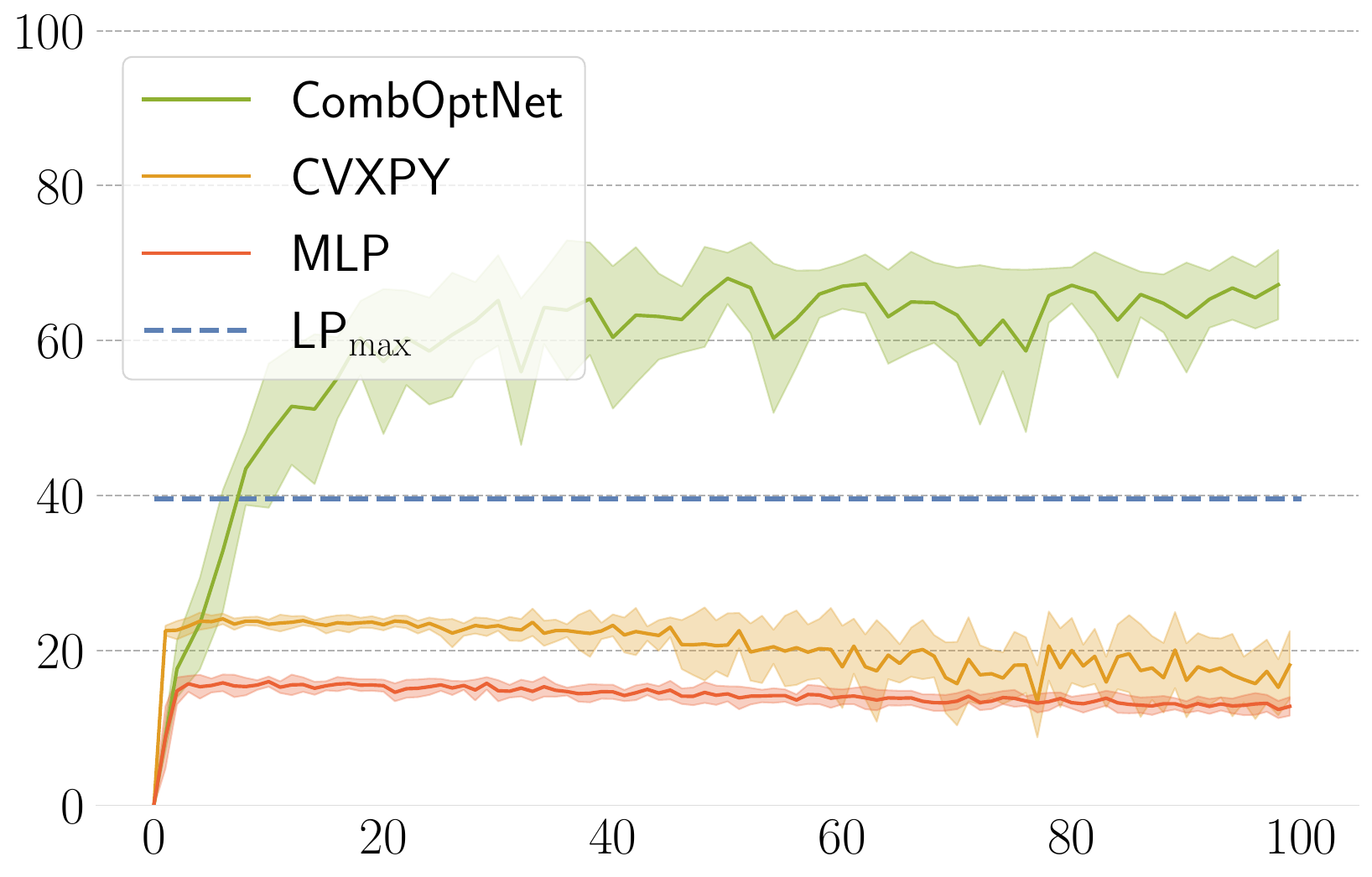}
		\caption{Evaluation accuracy ($\vy=\vy^*$ in \%) over training epochs.
		LP\textsubscript{max} is the maximum achievable LP relaxation accuracy.}
	\end{subfigure}
	\hfill
	\begin{subfigure}[t]{.43\textwidth}
		\centering
		\includegraphics[width=\linewidth]{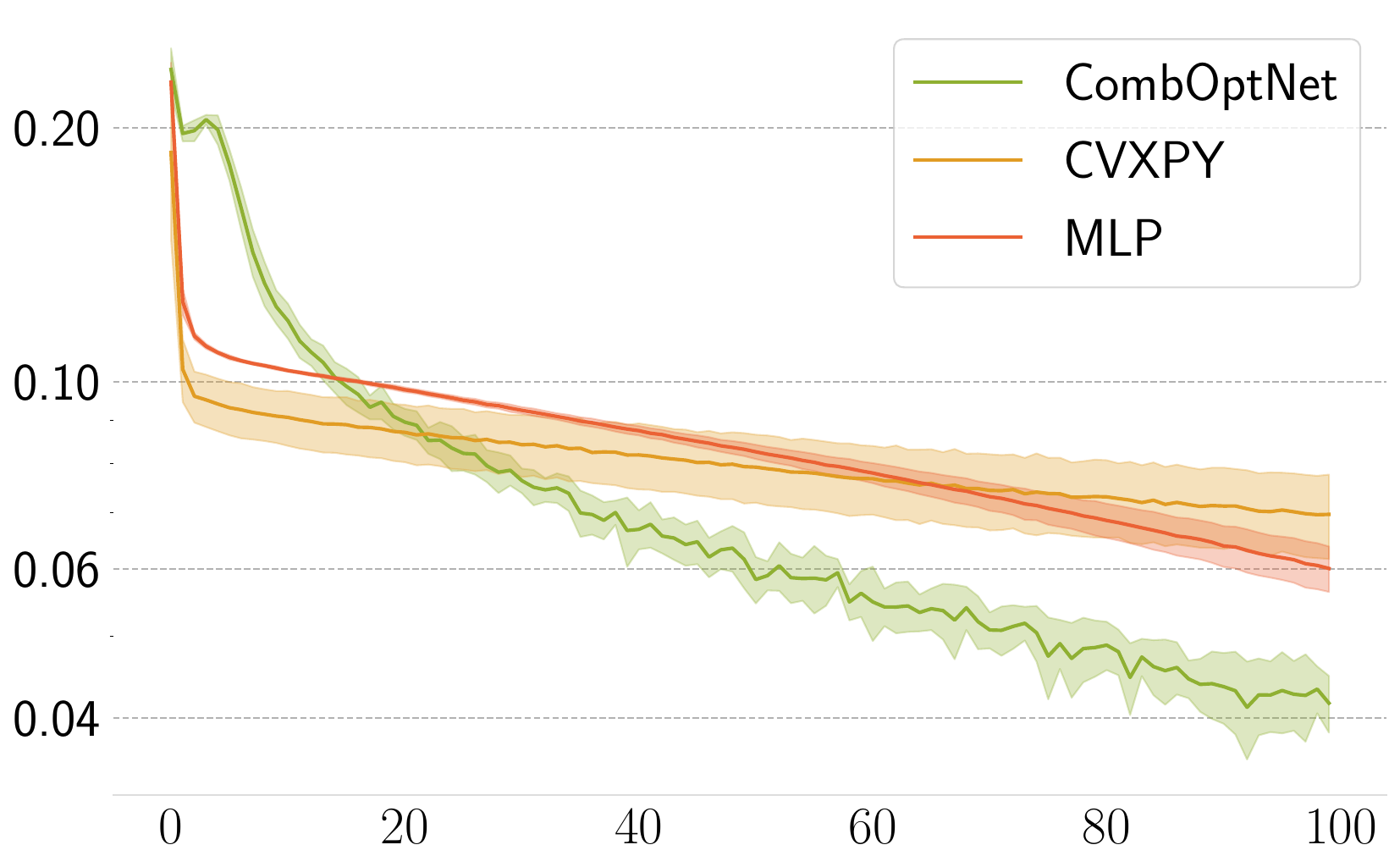}
		\caption{Training MSE loss over epochs.}
	\end{subfigure}
	\caption{Results or \knapsack demonstration.
	Reported error bars are over 10 restarts.
	}
\label{fig:knapsack-results}
\end{figure}

\subsection{Deep Keypoint Matching}

\paragraph{Problem formulation.}

Given are a source and target image showing an object of the same class (\eg \textit{airplane}), each labeled with a set of annotated keypoints (\eg \textit{left wing}).
The task is to find the correct matching between the sets of keypoints from visual information without access to the keypoint annotation.
As not every keypoint has to be visible in both images, some keypoints can also remain unmatched.

As in this task the combinatorial problem is known a priori, state-of-the-art methods are able to exploit this knowledge by using dedicated solvers.
However, we make the problem harder by omitting this knowledge.
Instead, we \emph{simultaneously} infer the problem specification \emph{and} train the feature extractor for the cost vector from data end-to-end.

\paragraph{Dataset.}

We use the SPair-71k dataset \cite{min2019spair71k} which was published in the context of dense image matching and was used as a benchmark for keypoint matching in recent literature \citep{rolinek2020deep}.
It includes 70\,958 image pairs prepared from Pascal VOC 2012 and Pascal 3D+ with rich pair-level keypoint annotations.
The dataset is split into 53\,340 training pairs, 5\,384 validation pairs and 12\,234 pairs for testing.

\begin{table}[tb]
\renewcommand{\arraystretch}{1.2}
\small
\centering
\caption{Results for the keypoint matching demonstration. Reported is the standard per-variable accuracy $(\%)$ metric over 5 restarts. Column $p\times p$ corresponds to matching $p$ source keypoints to $p$ target keypoints.
Original \bbgm has access to unary and quadratic costs; we also report performance with access only to unary costs as in~\method.
}
	\begin{tabular}{lcccc}
	\toprule
	\textbf{Method} & $4\times 4$ & $5\times 5$ & $6\times 6$ & $7\times 7$
		\\ \midrule
		\method & $83.1$ & $80.7$ & $78.6$ & $76.1$
		\\
		\bbgm (unary only) & $84.3$ & $81.6$ & $79.0$ & $76.5$
		\\
		\bbgm (unary \& quad.) & $84.3$ & $82.9$ & $80.5$ & $79.8$
		\\
	\bottomrule
	\end{tabular}
\label{tab:matching}
\end{table}

\paragraph{State-of-the-art.}
We compare to a state-of-the-art architecture \bbgm \citep{rolinek2020deep} that employs a dedicated solver for the quadratic assignment problem.
Given a pair of images and their corresponding sets of keypoint locations, the method constructs for each image a graph with the nodes corresponding to the keypoints. 
A network then predicts the unary and quadratic costs for matching the nodes and edges between the two graphs. 
A heavily optimized solver for the quadratic \textsc{assignment} problem then computes a consistent matching of the nodes.
When training the architecture, the incoming gradient is backpropagated through the solver to the unary and quadratic costs via blackbox backpropagation \citep{vlastelica2019differentiation}.

\paragraph{Architecture.}
We modify the \bbgm architecture by replacing the blackbox backpropagation module employing the dedicated solver with \method.

The drop-in replacement comes with a few important considerations.
Note that our method relies on a fixed dimensionality $n$ of the problem for learning a static (\ie not input-dependent) constraint set. 
However, in the \bbgm architecture the dimensionality of the predicted costs depends on the number of nodes (\ie keypoints) $p$ and number of edges in each of the two matched graphs.
As these quantities vary over the dataset, the dimensionality varies as well.
To avoid these issues, we resort to a simplified setting.

First, we fix the number of keypoints in both images.
For each $p=4,5,6,7$, we generate an augmented dataset by randomly removing additional keypoints in images with more than $p$ keypoints.
The images with fewer keypoints than $p$ are dropped.
Note that even for a fixed number of keypoints the number of edges (and hence the quadratic cost dimensionality) can vary.
Therefore, we omit the quadratic costs as an input to the ILP solver. 

These simplifications result in a fixed ILP dimensionality of $n=p^2$.
The number of learnable constraints coincides with the number of constraints in the corresponding unary \textsc{assignment} problem, \ie the combined number of keypoints in both images~$(m=2p)$.

The randomly initialized constraint set and the backbone architecture that produces the cost vectors are learned simultaneously from pairs of predicted solutions and \gt matchings using~\method.


\paragraph{Results.}
We report two results for the \bbgm architecture depending on the information available to the solver.
In the unary setting, the solver utilizes only unary costs, in the quadratic setting, it utilizes both unary and quadratic costs. The ILP solver is restricted to the unary setting.

Even though the task-agnostic \method is uninformed about the underlying combinatorial problem, its performance is very close to the privileged state-of-the-art method \bbgm, especially when \bbgm is restricted to use the same information (unary costs only).
These results are especially satisfactory, considering that \bbgm outperforms the previous state-of-the-art architecture \citep{fey2020deep} by several percentage points on experiments of this difficulty.
Example matchings are shown in \Figref{fig:matching-examples}. 


\begin{figure}[t]
 \setlength{\belowcaptionskip}{-0.5\baselineskip}
 \includegraphics[width=0.49\linewidth]{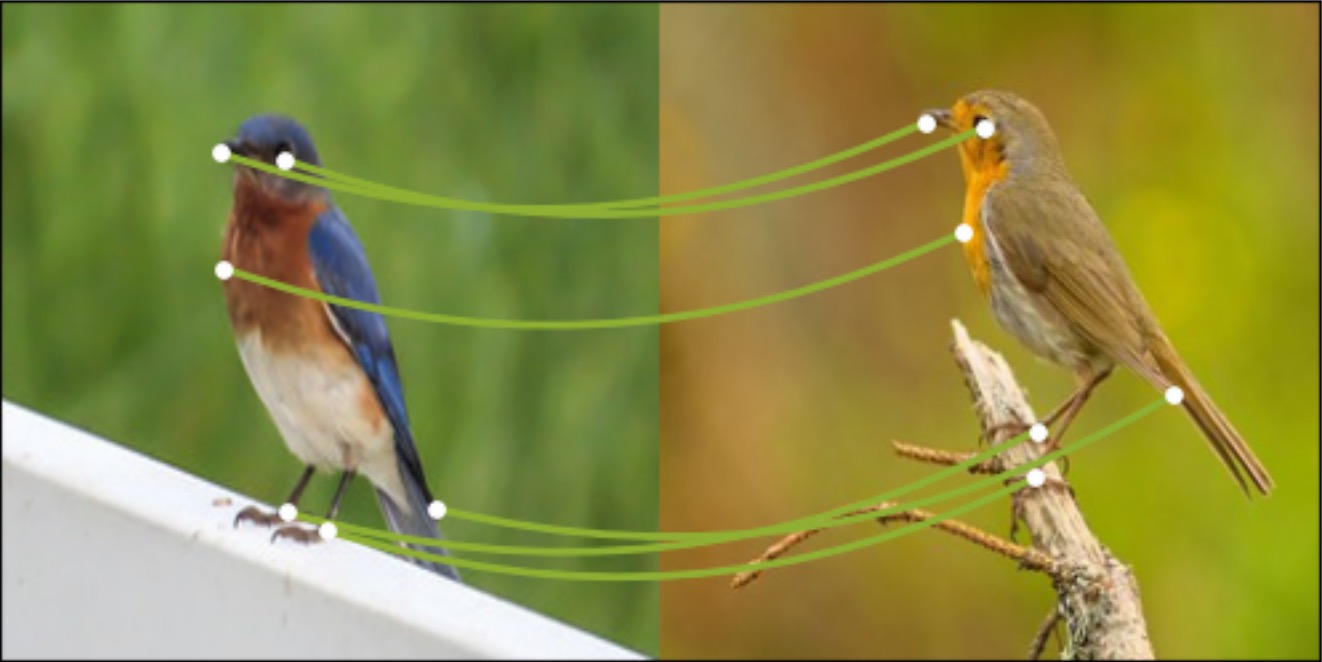}
 \hfil
 \includegraphics[width=0.49\linewidth]{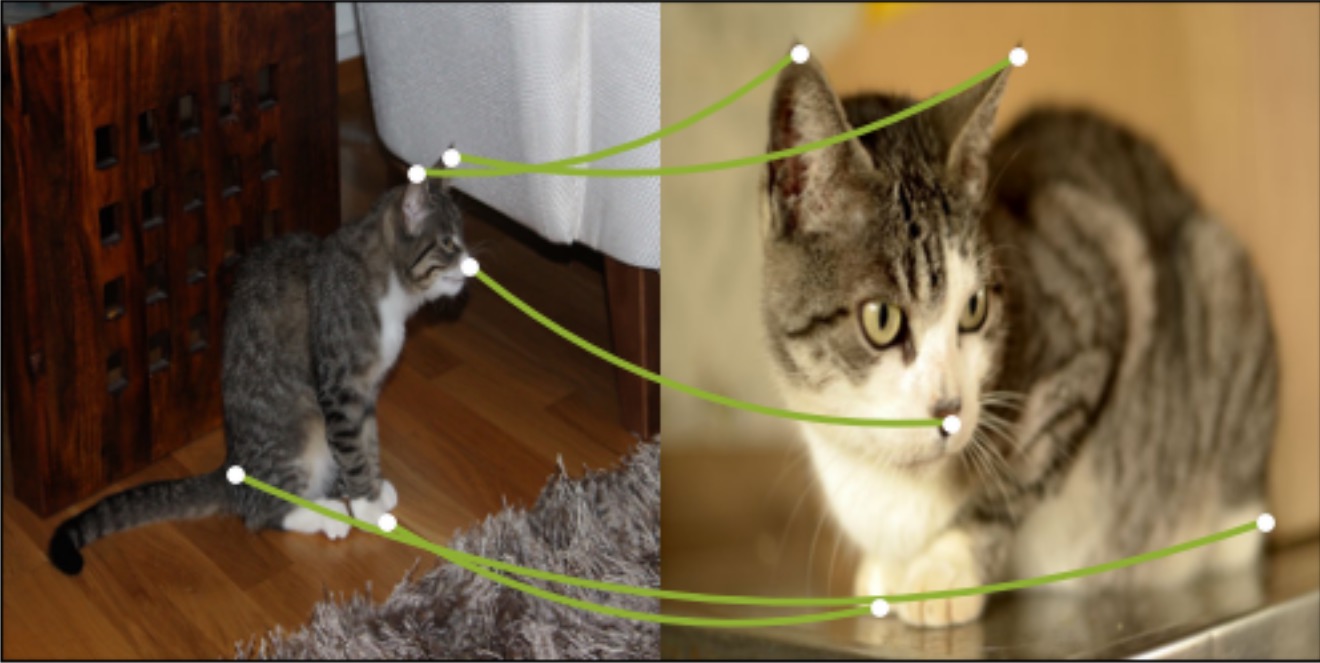}
 \\ [2pt]
 \includegraphics[width=0.49\linewidth]{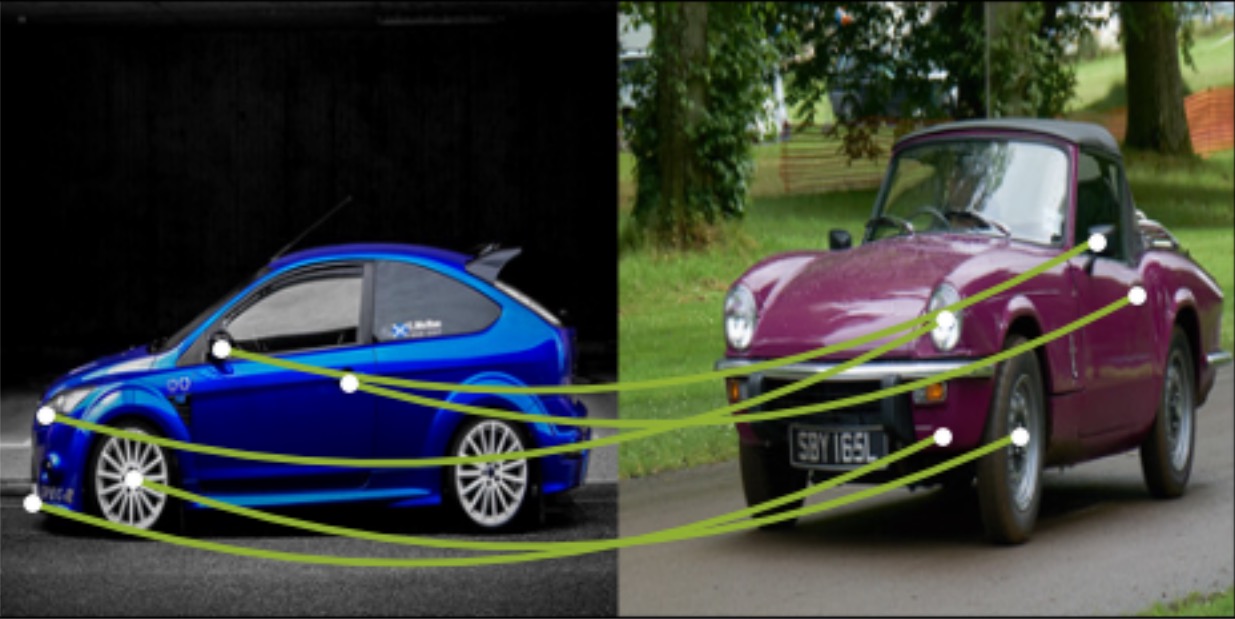}
 \hfil
 \includegraphics[width=0.49\linewidth]{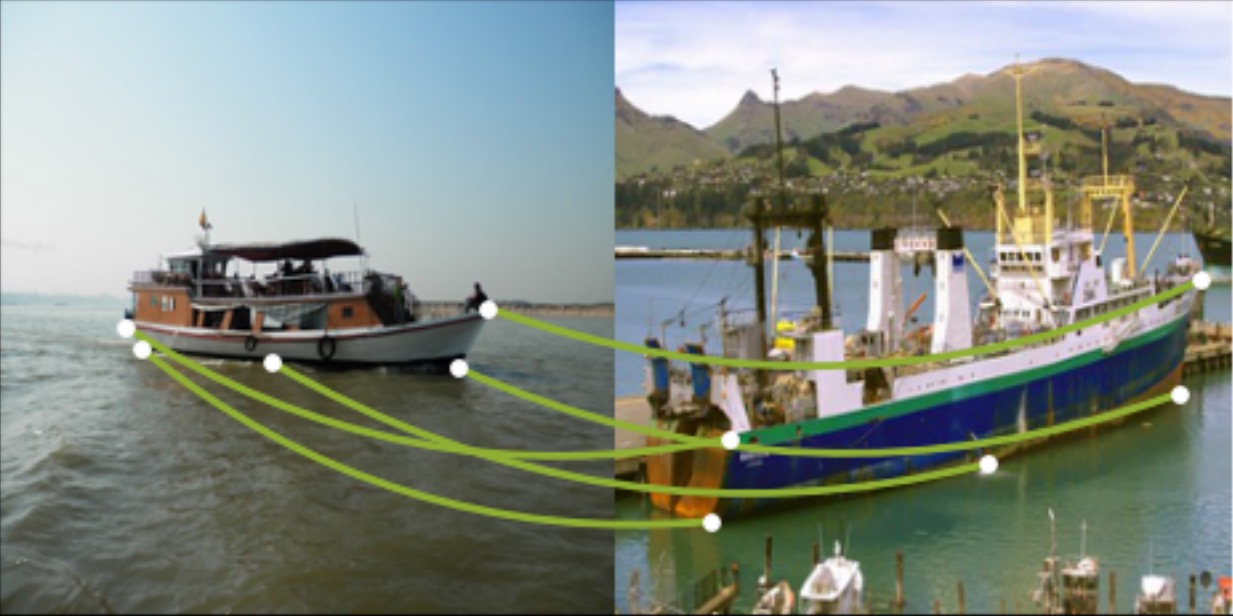}
 \\ [2pt]
 \includegraphics[width=0.49\linewidth]{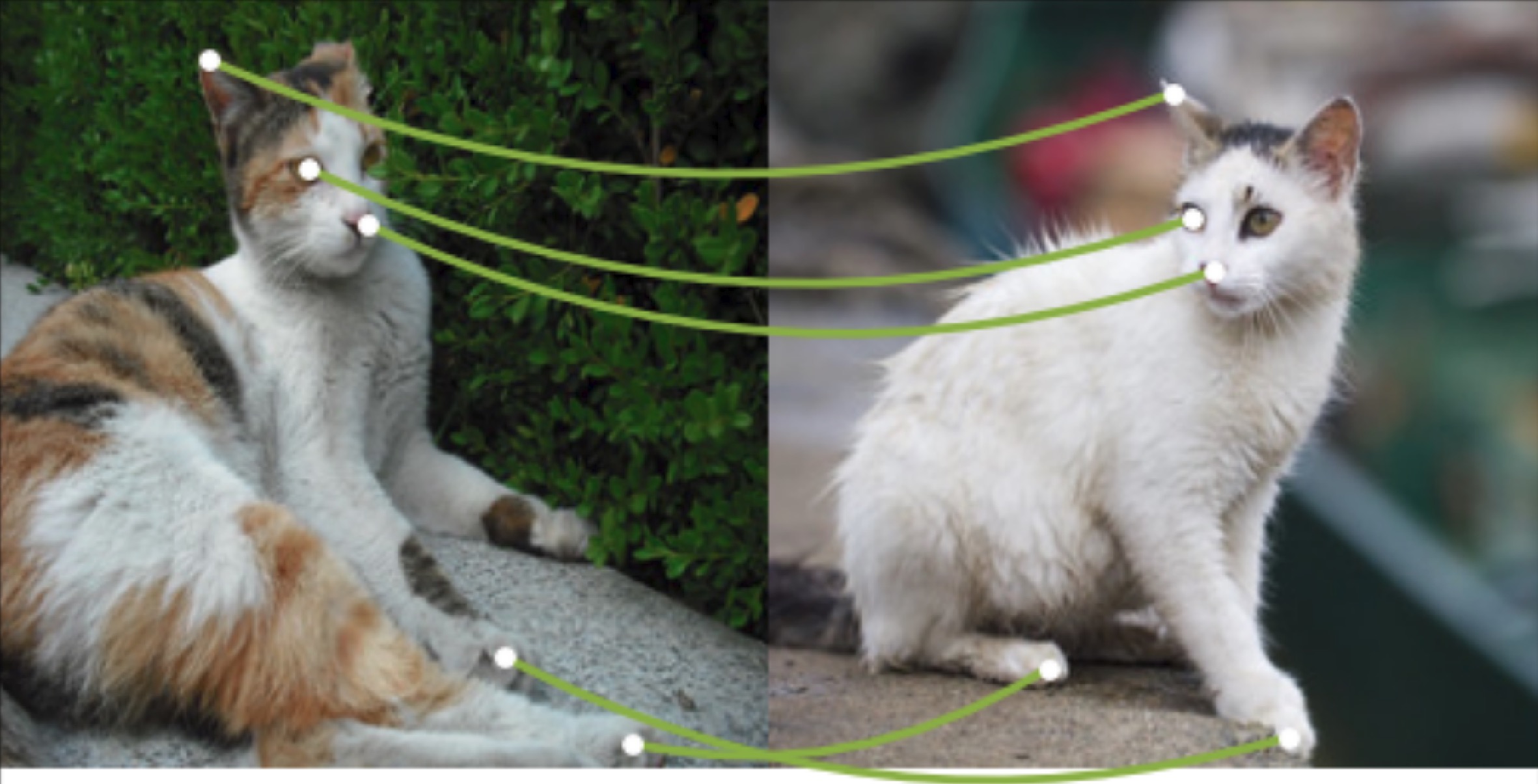}
 \hfil
 \includegraphics[width=0.49\linewidth]{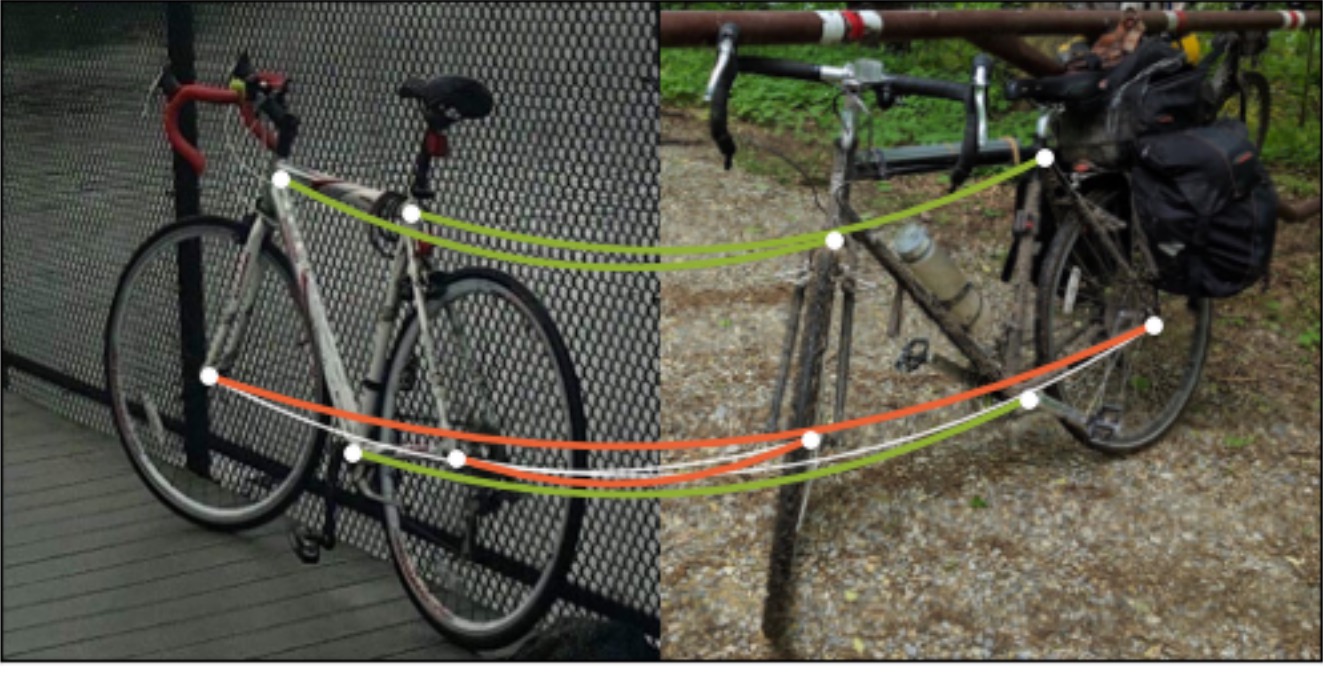}
 \\ [1.7pt]
 \includegraphics[width=0.49\linewidth]{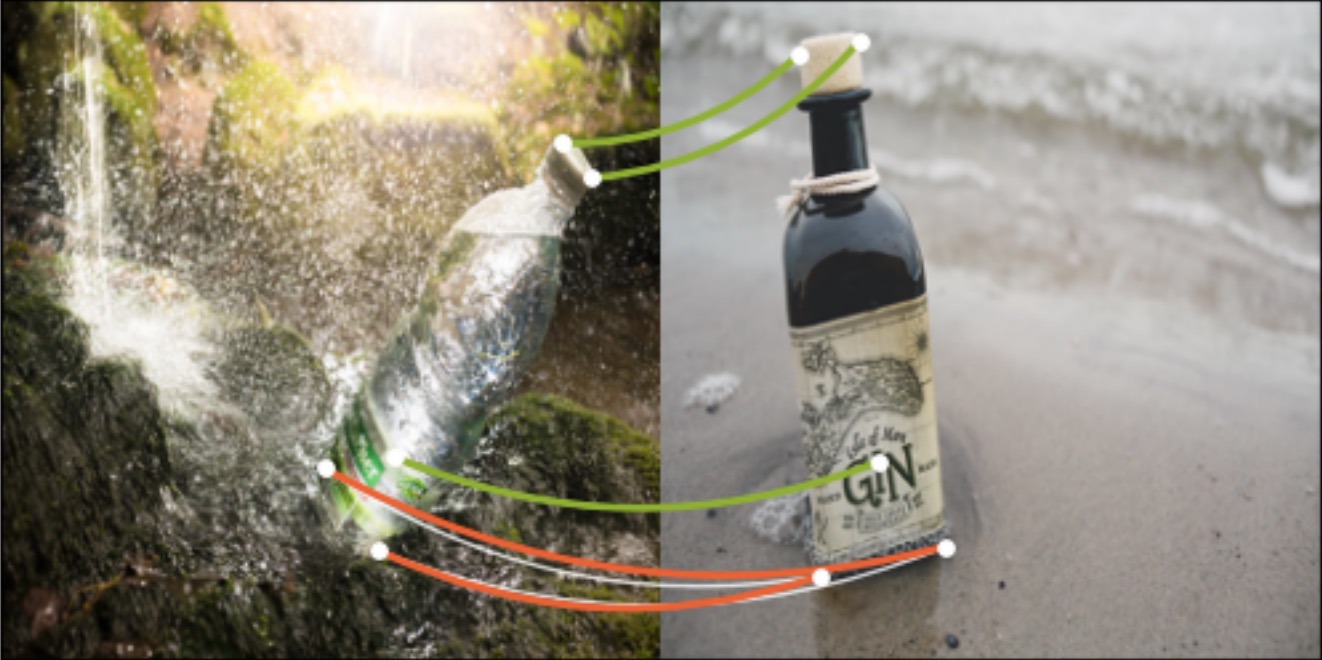}
 \hfil
 \includegraphics[width=0.49\linewidth]{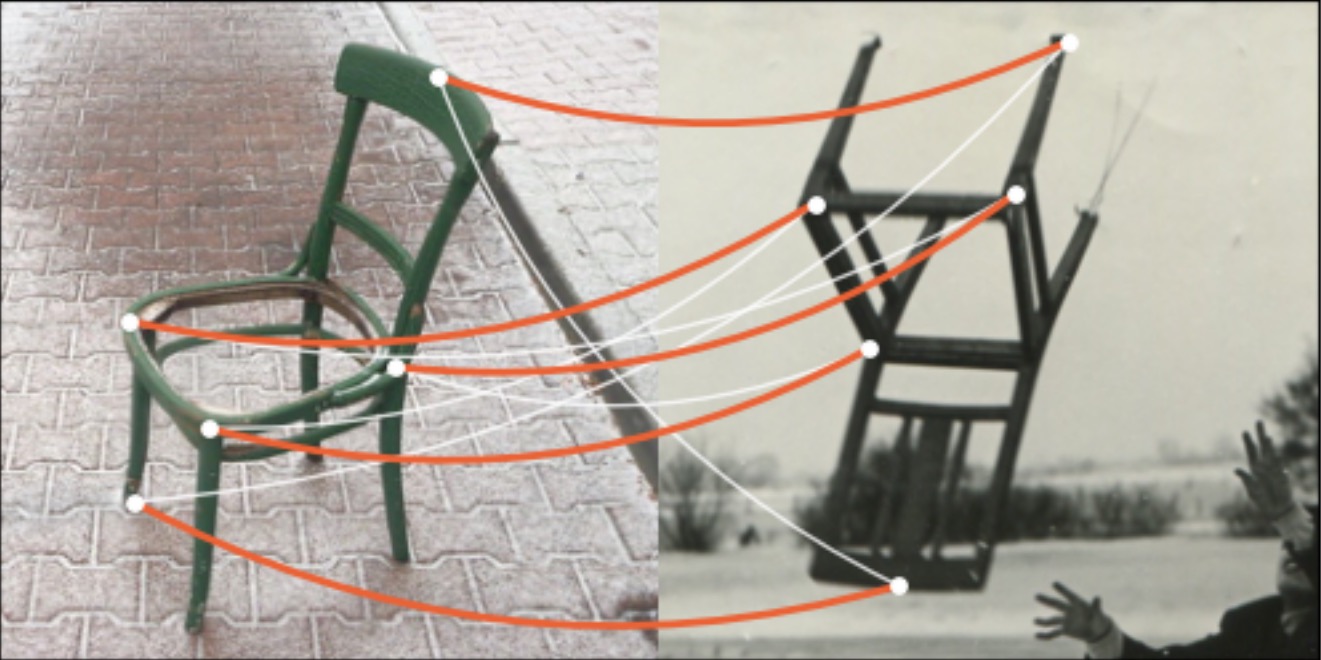}

 \caption{Example matchings predicted by \method.}
 \label{fig:matching-examples}
\end{figure}

\section{Conclusion}
We propose a method for integrating integer linear program solvers into neural network architectures as layers.
This is enabled by providing gradients for both the \emph{cost terms} and the \emph{constraints} of an ILP.
The resulting end-to-end trainable architectures are able to \emph{simultaneously} extract features from raw data and learn a suitable set of constraints that specify the combinatorial problem.
Thus, the architecture learns to fit the right NP-hard problem needed to solve the task.
In that sense, it strives to achieve universal combinatorial expressivity in deep networks---opening many exciting perspectives.

In the experiments, we demonstrate the flexibility of our approach,
using different input domains, natural language and images, and different combinatorial problems with the same \method module.
In particular, for combinatorially hard problems we see a strong advantage of the new architecture.

%
%
The potential of our method is highlighted by the demonstration on the keypoint matching benchmark.
Unaware of the underlying combinatorial problem, \method achieves a performance that is not far behind architectures employing dedicated state-of-the-art solvers.



In future work, we aim to make the number of constraints flexible and to explore more problems with hybrid combinatorial complexity and statistical learning aspects.

 %

\section*{Acknowledgements}
Georg Martius is a member of the Machine Learning Cluster of Excellence, funded by the Deutsche Forschungsgemeinschaft (DFG, German Research Foundation) under Germany’s Excellence Strategy -- EXC number 2064/1 -- Project number 390727645. We acknowledge the support from the German Federal Ministry of Education and Research (BMBF) through the Tübingen AI Center (FKZ: 01IS18039B).
This work was supported from Operational Programme Research, Development and Education -- Project Postdoc2MUNI (No.\ CZ.02.2.69/0.0/0.0/18\_053/0016952)

\bibliographystyle{icml2021}
\bibliography{arxiv_v2}

\appendix

\section{Demonstrations}

The code for all demonstrations is available at

\smallskip
\centerline{\href{https://github.com/martius-lab/CombOptNet}{github.com/martius-lab/CombOptNet}.}

\subsection{Implementation Details}

When learning multiple constraints, we replace the minimum in definition~\eqref{E:proxy-constraints} of mismatch function $P_{\Delta_k}$ with its softened version.
Therefore, not only the single closest constraint will shift towards $\vy'_k$, but also other constraints close to $\vy'_k$ will do.
For the softened minimum we use 
\begin{equation} \label{sup:E:softmin}
    \operatorname{softmin}(\vx)
			= -\tau\cdot \log\biggl(\sum_{k} \exp\left(-\frac{x_k}{\tau}\right)\biggr),
\end{equation}
which introduces the temperature $\tau$, determining the softening strength.

In all experiments, we normalize the cost vector $\vc$ before we forward it to the \method module.
For the loss we use the mean squared error between the normalized predicted solution $\vy$ and the normalized \gt solution $\vy^*$.
For normalization we apply the shift and scale that translates the underlying hypercube of possible solutions ($[0, 1]^n$ in \emph{binary} or $[-5, 5]^n$ in \emph{dense} case) to a normalized hypercube~$[-0.5, 0.5]^n$.


The hyperparameters for all demonstrations are listed in \Tabref{sup:tab:hypers-random-and-knapsack-and-matching}.
We use Adam \citep{kingma2014adam} as the optimizer for all demonstrations. Gurobi parameters for all experiments are kept to default.

\begin{table}[h]
    \small
    \centering
    \caption{Hyperparameters for all demonstrations.}
			\begin{tabular}{@{}lcccl@{}}
			\toprule
										& \textbf{WSC \& Random}      & \multirow{2}{*}{\textbf{Knapsack}} & \textbf{Keypoint}  \\
										& \textbf{Constraints} &                                    & \textbf{Matching}  \\ \midrule
			Learning rate & $5\times 10^{-4}$    & $5\times 10^{-4}$                  & $1\times 10^{-4}$   \\
			Batch size    & 8                    & 8                                  & 8                  \\
			Train epochs  & 100                  & 100                                & 10                 \\
			$\tau$        & 0.5                  & 0.5                                & 0.5                \\
			Backbone lr   & --                   & --                                 & $2.5\times 10^{-6}$ \\ \bottomrule
			\end{tabular}
    \label{sup:tab:hypers-random-and-knapsack-and-matching}
\end{table}

\paragraph{Random Constraints.}

For selecting the set of constraints for data generation, we uniformly sample constraint origins~$\vo_k$ in the center subcube (halved edge length) of the underlying hypercube of possible solutions.
The constraint normals $\va_k$ and the cost vectors $\vc$ are randomly sampled normalized vectors and the bias terms are initially set to $b_k=0.2$.
The signs of the constraint normals $\va_k$ are flipped in case the origin is not feasible, ensuring that the problem has at least one feasible solution.
We generate 10 such datasets for $m=1,2,4,8$ constraints in $n=16$ dimensions.
The size of each dataset is 1\,600 train instances and 1\,000~test instances.

For learning, the constraints are initialised in the same way except for the feasibility check, which is skipped since \method can deal with infeasible regions itself.

\paragraph{\textsc{Knapsack} from Sentence Description.}

Our method and CVXPY use a small neural network to extract weights and prices from the 4\,096-dimensional embedding vectors.
We use a two-layer MLP with a hidden dimension of 512, ReLU nonlinearity on the hidden nodes, and a sigmoid nonlinearity on the output.
The output is scaled to the \gt price range $[10,45]$ for the cost $\vc$ and to the \gt weight range $[15,35]$ for the constraint $\va$.
The bias term is fixed to the \gt knapsack capacity $b=100$.
Item weights and prices as well as the knapsack capacity are finally multiplied by a factor of $0.01$ to produce a reasonable scale for the constraint parameters and cost vector.

The CVXPY baseline implements a greedy rounding procedure to ensure the feasibility of the predicted integer solution with respect to the learned constraints.
Starting from the item with the largest predicted (noninteger) value, the procedure adds items to the predicted (integer) solution until no more items can be added without surpassing the knapsack capacity.

The MLP baseline employs an MLP consisting of three layers with dimensionality $100$ and ReLU activation on the hidden nodes.
Without using an output nonlinearity, the output is rounded to the nearest integer point to obtain the predicted solution.

\paragraph{Deep Keypoint Matching.}

We initialize a set of constraints exactly as in the \textit{binary} case of the Random Constraints demonstration.
We use the architecture described by \citet{rolinek2020deep}, only replacing the dedicated solver module with \method.

We train models for varying numbers of keypoints $p = 4,5,6,7$ in the source and target image, resulting in varying dimensionalities $n=p^2$
and number of constraints $m=2p$.
Consistent with \citet{rolinek2020deep}, all models are trained for 10 epochs, each consisting of 400 iterations with randomly drawn samples from the training set.
We discard samples with fewer keypoints than what is specified for the model through the dimensionality of the constraint set.
If the sample has more keypoints, we chose a random subset of the correct size.

After each epoch, we evaluate the trained models on the validation set.
Each model's highest-scoring validation stage is then evaluated on the test set for the final results.

\subsection{Runtime analysis.}

The runtimes of our demonstrations are reported in \Tabref{sup:tab:runtimes}.
Random Constrains demonstrations have the same runtimes as Weighted Set Covering since they share the architecture.

Unsurprisingly, \method has higher runtimes as it relies on ILP solvers which are generally slower than LP solvers.
Also, the backward pass of \method has negligible runtime compared to the forward-pass runtime.
In Random Constraints, Weighted Set Covering and \textsc{Knapsack} demonstration, the increased runtime is necessary, as the baselines simply do not solve a hard enough problem to succeed in the tasks.

In the Keypoint Matching demonstration, \method slightly drops behind \bbgm and requires higher runtime.
Such drawback is outweighed by the benefit of employing a broad-expressive model that operates without embedded knowledge of the underlying combinatorial task.



\begin{table}[t]
	\small
	\centering
	\caption{Average runtime for training and evaluating a model on a single Tesla-V100 GPU.
	For Keypoint Matching, the runtime for the largest model ($p=7$) is shown.}
	\begin{tabular*}{\linewidth}{l@{\extracolsep{\fill}}ccc}
				\toprule
							  & Weighted      & \multirow{2}{*}{Knapsack} & Keypoint \\
							  & Set Covering 	&                           & Matching \\ \midrule
				\method & 1h 30m        & 3h 50m                    & 5h 30m   \\
				CVXPY   & 1h            & 2h 30m                    & --       \\
				MLP     & 10m           & 20m                       & --       \\
				\bbgm   & --            & --                        & 55m      \\ \bottomrule
	\end{tabular*}
	\label{sup:tab:runtimes}
\end{table}

\subsection{Additional Results}

\paragraph{Random Constraints \& Weighted Set Covering.}

We provide additional results regarding the increased amount of learned constraints in \Tabref{sup:tab:additional-results} and \ref{sup:tab:additional-results:setcovering})
and the choice of the loss function \Tabref{sup:tab:additional-results:loss}. 

\renewcommand{\thefootnote}{$\star$}
\begin{table}[b]
    \small
    \centering
    \caption{
		Random Constraints demonstration with multiple learnable constraints.
		Using a dataset with $m$ \gt constraints, we train a model with $k \times m$ learnable constraints.
		Reported is evaluation accuracy ($\vy=\vy^*$ in \%) for $m=1,2,4,8$.
		Statistics are over 20 restarts (2 for each of the 10 dataset seeds).}
		\begin{adjustbox}{max width=\linewidth}
    	\begin{tabular}{@{}rldddd@{\hskip1.1ex}}
    	\toprule
				& \multicolumn{1}{c}{$m$} & \multicolumn{1}{c}{1} & \multicolumn{1}{c}{2} & \multicolumn{1}{c}{4} & \multicolumn{1}{c}{8}
    		\\ \midrule
    		\multirow{3}{1ex}{\rotatebox{90}{\strut\bin}}
				& $1\times m$\footnotemark[1]  & \bf 97.8 + \bf 0.7 & 94.2 + 10.1 & 77.4 + 13.5 & 46.5 + 12.4 \\
    		& $2\times m$                  & 97.3 + 0.9 & 95.1 + 1.6  & 87.8 + 5.2  & 63.1 + 7.0 \\
    		& $4\times m$                  & 96.9 + 0.7 & \bf 95.1 + \bf 1.2  & \bf 88.7 + \bf 2.3  & \bf 77.7 + \bf 3.2 \\
				\midrule
    		\multirow{3}{1ex}{\rotatebox{90}{\strut\den}}
    		& $1\times m$\footnotemark[1]  & 87.3 + 2.5 & 70.2 + 11.6 & 29.6 + 10.4 & 2.3 + 1.2 \\
    		& $2\times m$                  & \bf 87.8 + \bf 1.7 & \bf 73.4 + \bf 2.4  & \bf 32.7 + \bf 7.6  & 2.4 + 0.8 \\
    		& $4\times m$                  & 85.0 + 2.6 & 64.6 + 3.9  & 28.3 + 2.7  & \bf 2.9 + \bf 1.3 \\
    	\bottomrule
    	\end{tabular}
		\end{adjustbox}
    \label{sup:tab:additional-results}
\end{table}
\footnotetext[1]{Used in the main demonstrations.} 

With a larger set of learnable constraints the model is able to construct a more complex feasible region. While in general this tends to increase performance and reduce variance by increasing robustness to bad initializations, it can also lead to overfitting similarly to a neural net with too many parameters. 

\begin{table}[t]
    \small
    \centering
    \caption{
		Weighted set covering demonstration with multiple learnable constraints.
		}
    	\begin{tabular*}{\linewidth}{@{}l@{\extracolsep{\fill}}dddd@{\hskip 1ex}}
    	\toprule
				$k$ & \multicolumn{1}{c}{4} & \multicolumn{1}{c}{6} & \multicolumn{1}{c}{8} & \multicolumn{1}{c}{10}
				\\ \midrule
    		$1$\footnotemark[1]  & 100 + 0.0 & 97.2 + 6.4 & 79.7 + 12.1 & 56.7 + 14.8 \\
    		$2$                  & 100 + 0.0 & 99.5 + 1.9 & \bf 99.3 + \bf 0.8  & 80.4 + 13.0 \\
    		$4$                  & 100 + 0.0 & \bf 99.9 + \bf 0.0 & 97.9 + 6.4  & \bf 85.2 + \bf 8.1  \\
    	\bottomrule
    	\end{tabular*}
    \label{sup:tab:additional-results:setcovering}
\end{table}

In the \den case, we also compare different loss functions which is possible because \method can be used as an arbitrary layer.
As shown in \Tabref{sup:tab:additional-results:loss}, this choice matters, with the MSE loss, the L1 loss and the Huber loss outperforming the L0 loss. This freedom of loss function choice can prove very helpful for training more complex architectures.

\begin{table}[h]
    \small
    \centering
		\captionsetup{belowskip=-\baselineskip}
    \caption{
		Random Constraints \den demonstration with various loss functions.
		For the Huber loss we set $\beta=0.3$.
		Statistics are over 20 restarts (2 for each of the 10 dataset seeds).}
    	\begin{tabular*}{\linewidth}{@{\extracolsep{\fill}}ldddd@{}}
    	\toprule
				Loss & \multicolumn{1}{c}{1} & \multicolumn{1}{c}{2} & \multicolumn{1}{c}{4} & \multicolumn{1}{c}{8}
				\\ \midrule
    		MSE\footnotemark[1]    & 87.3 + 2.5 & 70.2 + 11.6 & 29.6 + 10.4 & 2.3 + 1.2 \\
    		Huber                  & 88.3 + 4.0 & 75.4 + 9.3  & 25.0 + 11.8 & \bf 2.6 + \bf 2.7 \\
    		L0	                   & 85.9 + 3.4 & 65.8 + 3.5  & 15.3 + 4.3  & 1.1 + 0.3 \\
    		L1	                   & \bf 89.2 + \bf 1.6 & \bf 75.8 + \bf 10.8 & \bf 30.2 + \bf 16.5 & 2.1 + 1.2 \\
    	\bottomrule
    	\end{tabular*}
    \label{sup:tab:additional-results:loss}
\end{table}

\paragraph{\textsc{Knapsack} from Sentence Description.}

As for the Random Constraints demonstration, we report the performance of \method on the \textsc{Knapsack} task for a higher number of learnable constraints.
The results are listed in \Tabref{sup:tab:additional-results:knapsack}. 
Similar to the \bin Random Constraints ablation with $m=1$, increasing the number of learnable constraints does not result in strongly increased performance.

\begin{table}[b]
    \small
    \centering
    \caption{
		Knapsack demonstration with more learnable constraints.
		Reported is evaluation accuracy ($\vy=\vy^*$ in \%) for $m=1,2,4,8$ constraints.
		Statistics are over 10 restarts.}
    	\begin{tabular}{dddd}
    	\toprule
				\multicolumn{1}{c}{1\footnotemark[1]} & \multicolumn{1}{c}{2} & \multicolumn{1}{c}{4} & \multicolumn{1}{c}{8}
    		\\ \midrule
    		64.7 + 2.8 & 63.5 + 3.7 & \bf 65.7 + \bf 3.1 & 62.6 + 4.4 \\
    	\bottomrule
    	\end{tabular}
    \label{sup:tab:additional-results:knapsack}
\end{table}

Additionally, we provide a qualitative analysis of the results on the \textsc{Knapsack} task.
In \Figref{sup:fig:additional-results:knapsack:prices} we compare the total \gt price of the predicted instances to the total price of the \gt solutions on a single evaluation of the trained models. 

The plots show that \method is achieving much better results than CVXPY.
The total prices of the predictions are very close to the optimal prices and only a few predictions are infeasible,
while CVXPY tends to predict infeasible solutions and only a few predictions have objective values matching the optimum.

In \Figref{sup:fig:additional-results:knapsack:errors} we compare relative errors on the individual item weights and prices on the same evaluation of the trained models as before.
Since (I)LP costs are scale invariant, we normalize predicted price vector to match the size of the \gt price vector before the comparison.

\method shows relatively small normally distributed errors on both prices and weights, precisely as expected from the prediction of a standard network.
CVXPY reports much larger relative errors on both prices and weights (note the different plot scale).
The vertical lines correspond to the discrete steps of \gt item weights in the dataset.
Unsurprisingly, the baseline usually tends to either overestimate the price and underestimate the item weight, or vice versa, due to similar effects of these errors on the predicted solution.

\begin{figure}[t]
  \centering
  \begin{subfigure}{.52\linewidth}
		\centering
		\includegraphics[width=\linewidth]{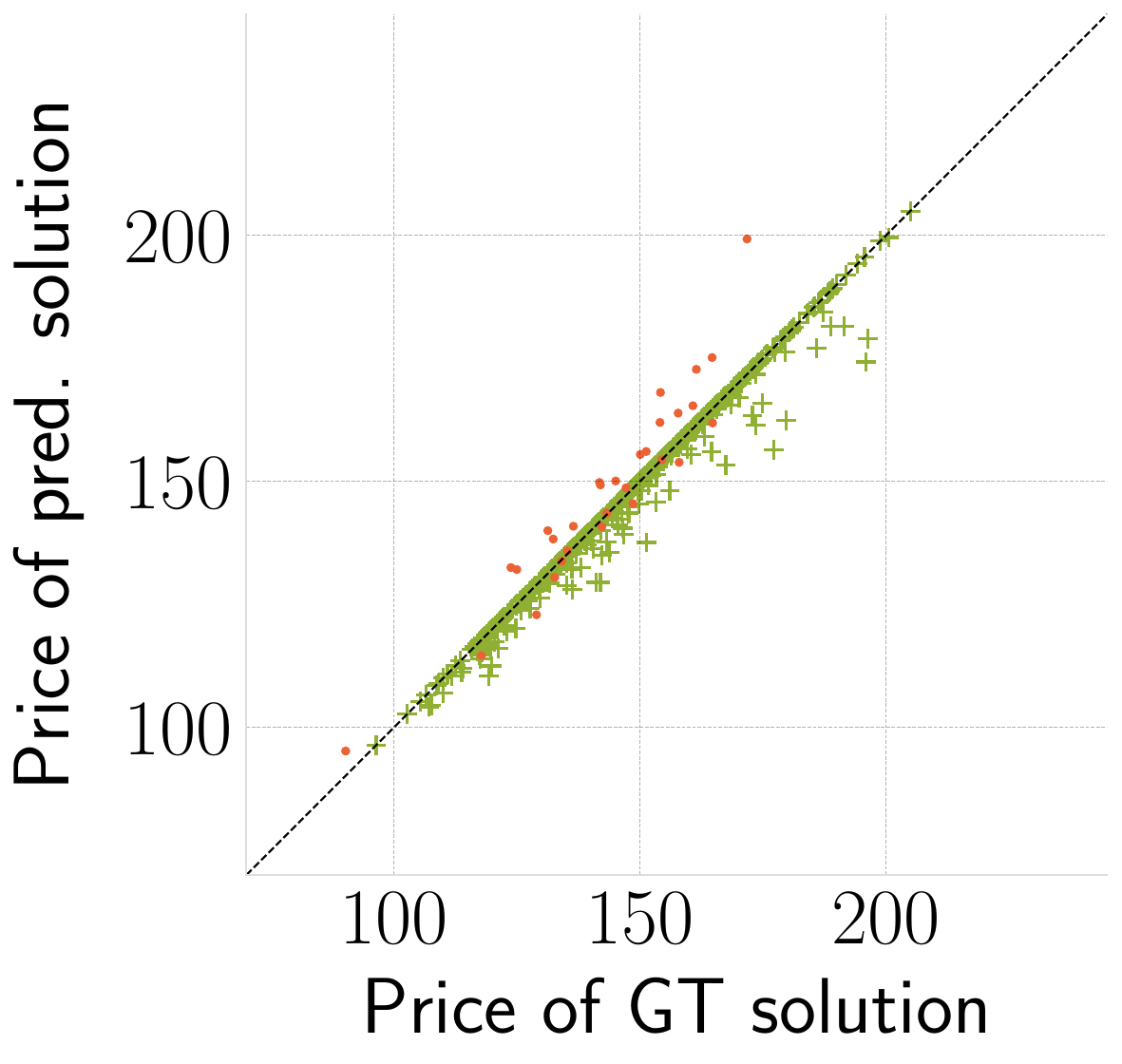}
		\caption{\method}
	\end{subfigure}
  \begin{subfigure}{.47\linewidth}
		\centering
		\includegraphics[width=\linewidth]{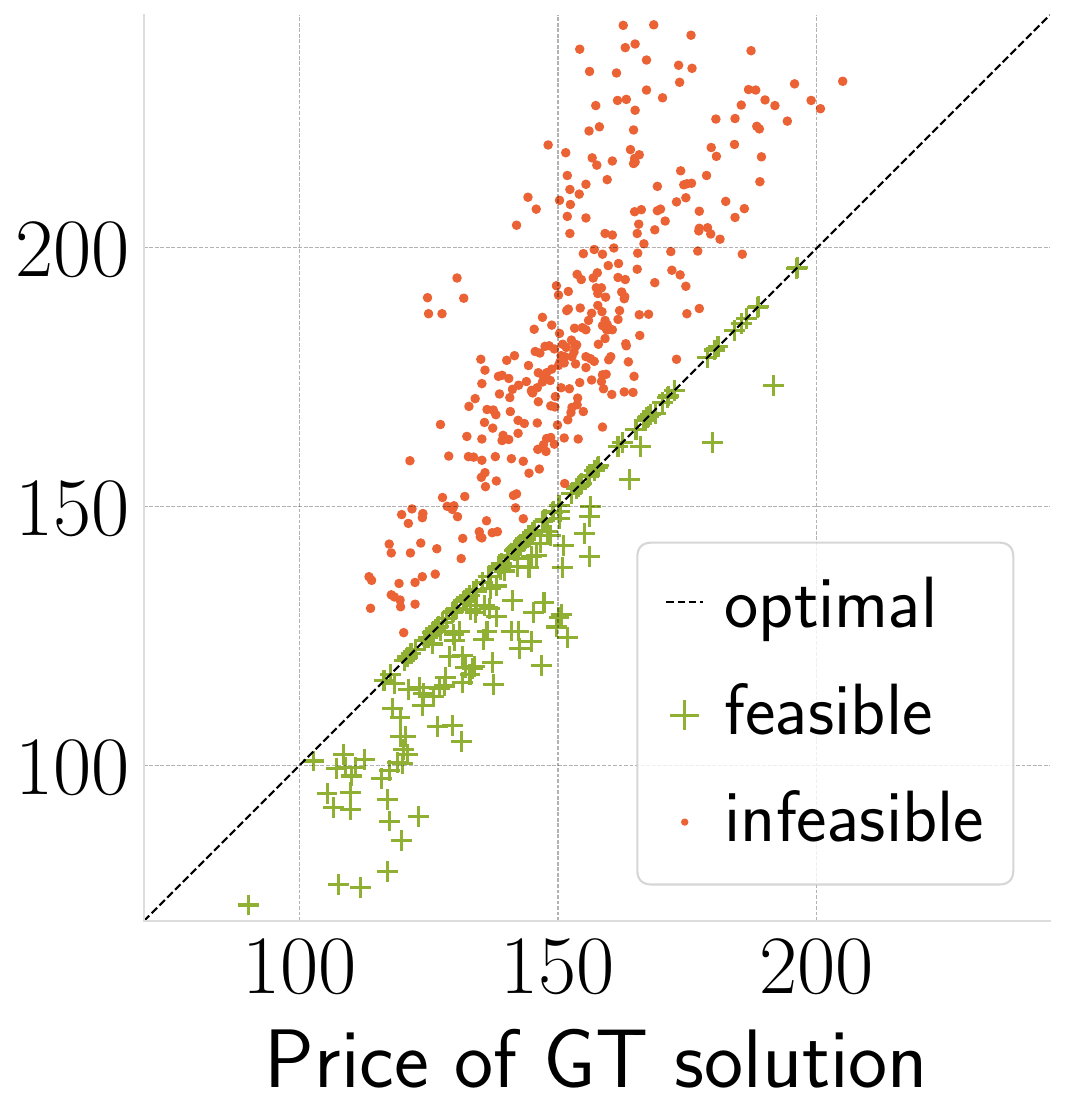}
		\caption{CVXPY}
  \end{subfigure}
  \caption{Prices analysis for the \textsc{Knapsack} demonstration.
	For each test set instance, we plot the total price of the predicted solution over the total price of the \gt solution.
	Predicted solutions which total weight exceeds the knapsack capacity are colored in red (cross).
	}
  \label{sup:fig:additional-results:knapsack:prices}
\end{figure}

\begin{figure}[h]
  \centering
  \begin{subfigure}{0.51\linewidth}
		\centering
		\includegraphics[width=\linewidth]{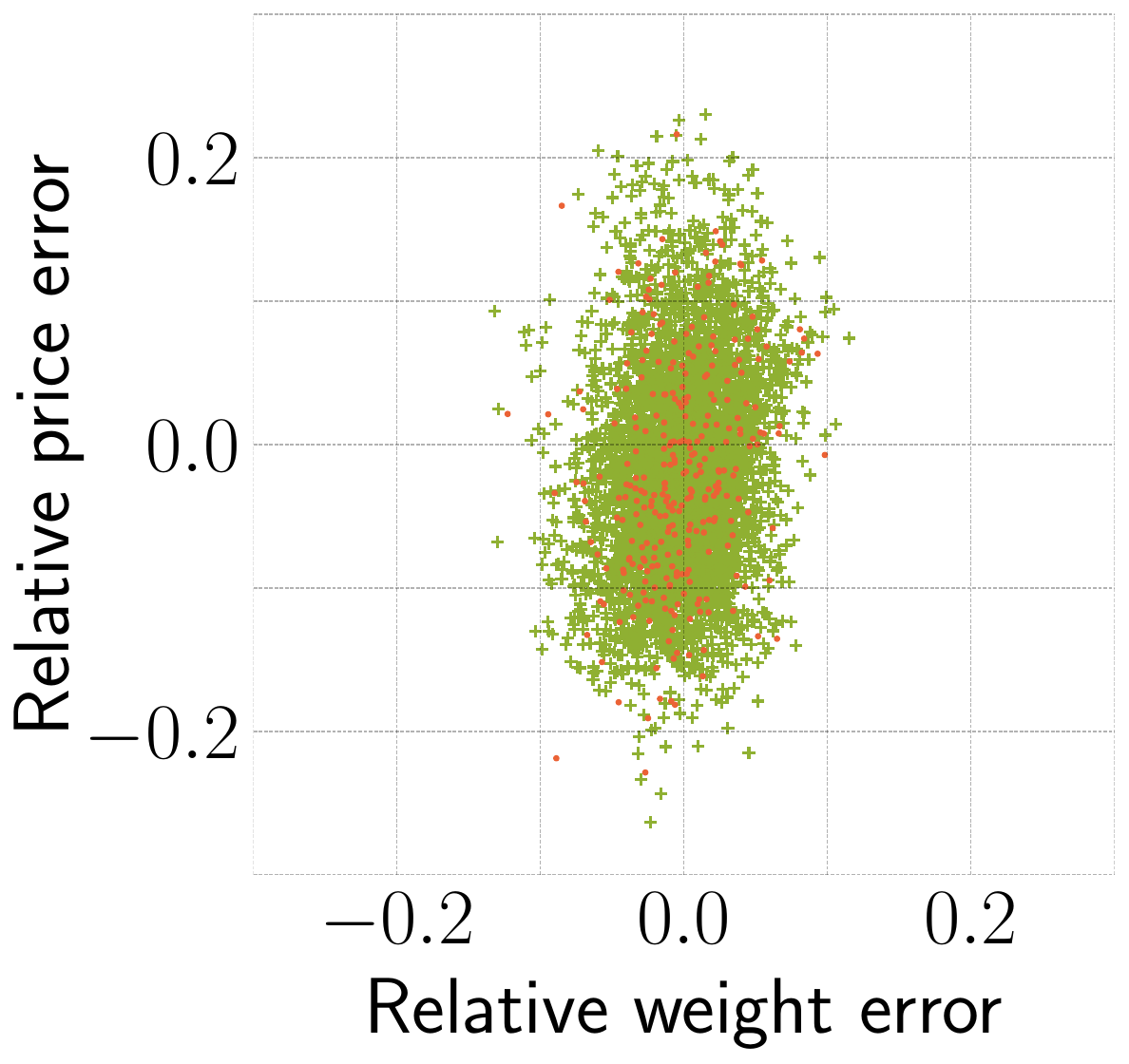}
		\caption{\method}
	\end{subfigure}
  \begin{subfigure}{.48\linewidth}
		\centering
		\includegraphics[width=\linewidth]{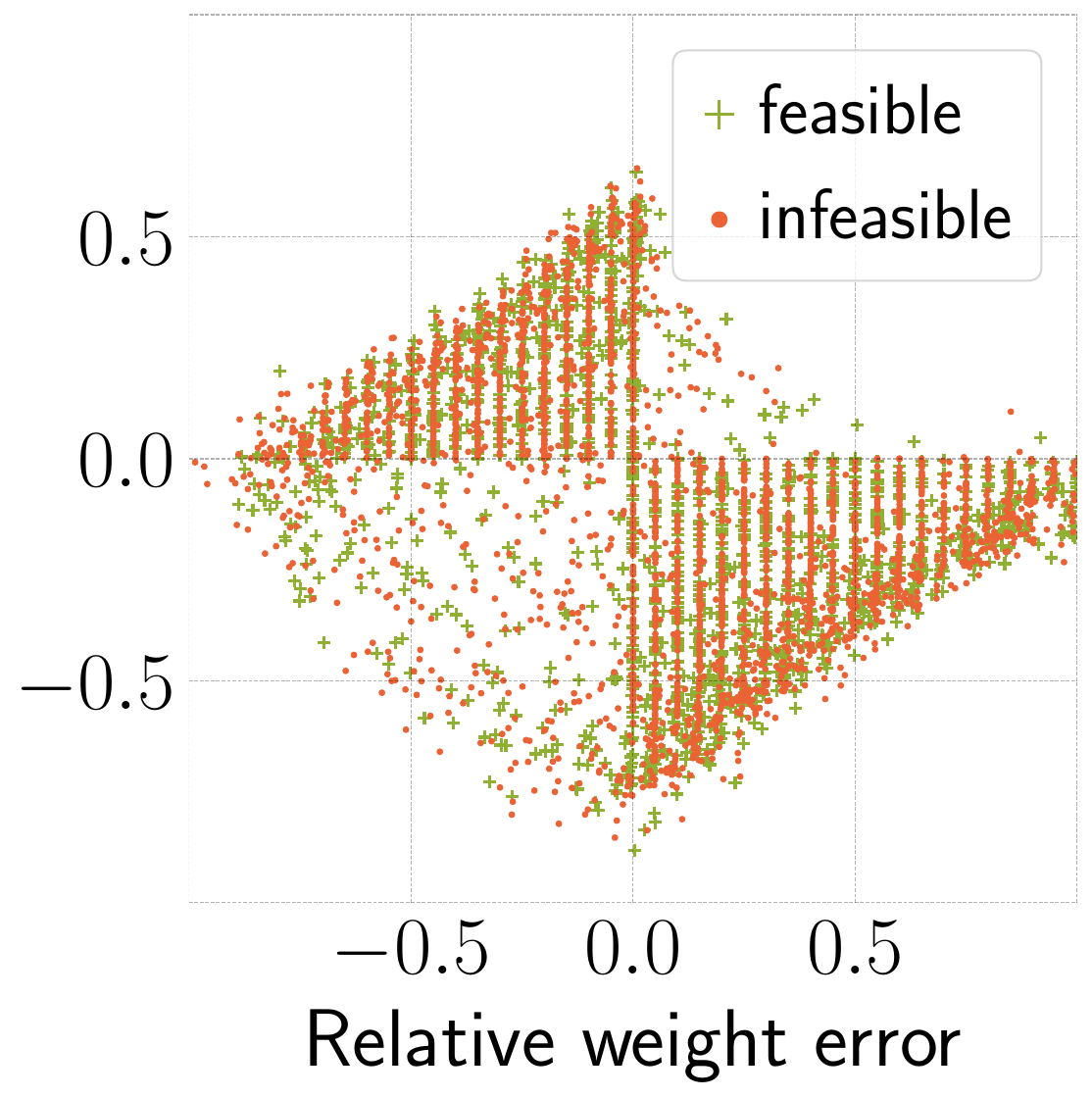}
		\caption{CVXPY}
  \end{subfigure}
  \caption{Qualitative analysis of the errors on weights and prices in the \textsc{Knapsack} demonstration.
	We plot the relative error between predicted and \gt item prices over the relative error between predicted and \gt item weights.
	Colors denote whether the predicted solution is feasible in terms of \gt weights.}
  \label{sup:fig:additional-results:knapsack:errors}
\end{figure}

\subsection{Ablations}

We ablate the choices in our architecture and model design on the Random Constraints (RC) and Weighted Set Covering (WSC) tasks.
In \Tabref{sup:tab:ablations:random-const} and \ref{sup:tab:ablations:set-covering} we report constraint parametrization, choice of basis, and minima softening ablations.

The ablations show that our parametrization with learnable origins is consistently among the best ones.
Without learnable origins, the performance is highly dependend on the origin of the coordinate system in which the directly learned parameters $(\mA,\vb)$ are defined.

The choice of basis in the gradient decomposition shows a large impact on performance. 
Our basis $\Delta$ \eqref{E:basis} is outperforming the canonical one in the \bin \rc and \wsc demonstration, while showing performance similar to the canonical basis in the \den \rc case.
The canonical basis produces directions for the computation of $\vy_k'$ that in many cases point in very different directions than the incoming descent direction.
As a result, the gradient computation leads to updates that are very detached from the original incoming gradient. 

Finally, the softened minimum leads to increased performance in all demonstrations. 
This effect is apparent particularly in the case of a binary solution space, as the constraints can have a relevant impact on the predicted solution $\vy$ over large distances.
Therefore, only updating the constraint which is closest to the predicted solution $\vy$, as it is the case for a hard minimum, gives no gradient to constraints that may potentially have had a huge influence on~$\vy$.

\section{Method}

\subsection{Fixed Constraints}
In the paper we omitted the discussion of the effect of fixed (\ie not learnable) constraints on the presented update rules.
In our experiments, we assumed fixed feasible region to be a hypercube $[l, u]^n$ with $l,u\in\sZ, l<z$.
Fixed constraints can cause problems, as our decomposition only guarantees that the targets $\vy'_k$ are integer points, but their feasibility \wrt to the fixed constraints is not guaranteed.
Therefore, our assumption of $\vy'_k$ being an \emph{attainable} target can be violated.

\begin{table*}[t]
		\vspace{2\baselineskip}
    \centering
		\captionsetup{width=.85\linewidth}
    \caption{Ablations of \method on Random Constraints demonstration.
		Reported is evaluation accuracy ($\vy=\vy^*$ in \%) for $m=1,2,4,8$ \gt constraints.
		Statistics are over 20 restarts (2 for each of the 10 dataset seeds).}
    	\begin{tabular*}{.85\linewidth}{@{\extracolsep{\fill}}ccle{3,2}dde{3,4}}
    	\toprule
    		&& Method & \multicolumn{1}{c}{1} & \multicolumn{1}{c}{2} & \multicolumn{1}{c}{4} & \multicolumn{1}{c}{8}
    		\\ \midrule
				\multirow{8}{1ex}{\rotatebox{90}{\strut\bin}}
				&\multirow{3}{1ex}{\rotatebox{90}{\strut param.}}
    		&  learnable origins\footnotemark[1]          & \bf 97.8 + \bf 0.7 & 94.2 + 10.1 & \bf 77.4 + \bf 13.5 & \bf 46.5 + \bf 12.4 \\
    		&& direct (origin at corner)                  & 97.4 + 1.0 & \bf 94.9 + \bf 7.0  & 59.0 + 26.8 & 26.9 + 10.3 \\
    		&& direct (origin at center)                 & 98.0 + 0.5 & 97.1 + 0.6  & 70.5 + 19.1 & 44.6 + 5.9  \\
				\cmidrule(l){2-7}
				&\multirow{2}{1ex}{\rotatebox{90}{\strut basis}}
    		&  $\Delta$ basis\footnotemark[1]             & \bf 97.8 + \bf 0.7 & \bf 94.2 + \bf 10.1 & \bf 77.4 + \bf 13.5 & \bf 46.5 + \bf 12.4 \\
    		&& canonical                                  & 96.3 + 1.9 & 70.8 + 4.1  & 14.4 + 3.2  & 2.7  + 0.9  \\
				\cmidrule(l){2-7}
				&\multirow{3}{1ex}{\rotatebox{90}{\strut min}}
    		&  hard                                       & & 83.1 + 13.2 & 55.4 + 18.9 & 37.7 + 8.7  \\ 
    		&& soft ($\tau=0.5$)\footnotemark[1]          & 97.8 + 0.7 & 94.2 + 10.1 & \bf 77.4 + \bf 13.5 & \bf 46.5 + \bf 12.4 \\
    		&& soft ($\tau=1.0$)                          & & \bf 95.7 + \bf 2.2  & 70.2 + 14.1 & 36.0 + 9.7  \\ 

				\midrule
				\multirow{8}{1ex}{\rotatebox{90}{\strut\den}}
        &\multirow{3}{1ex}{\rotatebox{90}{\strut param.}}
    		&  learnable origins\footnotemark[1]          & \bf 87.3 + \bf 2.5 & 70.2 + 11.6 & 29.6 + 10.4 & 2.3 + 1.2 \\
    		&& direct (origin at corner)                  & 86.7 + 3.0 & \bf 74.6 + \bf 3.6  & \bf 32.6 + \bf 13.7 & \bf 2.8 + \bf 0.5 \\
    		&& direct (origin at center)                 & 83.0 + 6.1 & 43.8 + 13.2 & 11.6 + 3.1  & 1.1 + 0.5 \\
				\cmidrule(l){2-7}
        &\multirow{2}{1ex}{\rotatebox{90}{\strut basis}}
    		&  $\Delta$ basis\footnotemark[1]             & 87.3 + 2.5 & 70.2 + 11.6 & \bf 29.6 + 10.4 & 2.3 + 1.2 \\
    		&& canonical                                  & \bf 88.6 + \bf 1.4 & \bf 71.6 + \bf 1.6  & 26.8 + 4.1  & \bf 4.0 + \bf 0.7 \\
				\cmidrule(l){2-7}
        &\multirow{3}{1ex}{\rotatebox{90}{\strut min}}
    		&  hard                                       &  & 70.8 + 15.1 & 21.4 + 10.7 & 2.2 + 2.1 \\ 
    		&& soft ($\tau=0.5$)\footnotemark[1]          & 89.1 + 2.8 & 70.2 + 11.6 & 29.6 + 10.4 & \bf 2.3 + \bf 1.2 \\ 
    		&& soft ($\tau=1.0$)                          &  & \bf 73.0 + \bf 12.1 & \bf 31.9 + \bf 11.7 & 2.2 + 1.5 \\ 
    	\bottomrule
    	\end{tabular*}
    \label{sup:tab:ablations:random-const}
\end{table*}

\begin{table*}[h!]
		\vspace{3\baselineskip}
    \centering
		\captionsetup{width=.85\linewidth}
    \caption{Ablations of \method on Weighted Set Covering.
		Reported is evaluation accuracy ($\vy=\vy^*$ in \%) for $m=4,6,8,10$ \gt constraints.
		Statistics are over 20 restarts (2 for each of the 10 dataset seeds).}
    	\begin{tabular*}{.85\linewidth}{@{\extracolsep{\fill}}clddde{3,4}}
    	\toprule
    	& Method & \multicolumn{1}{c}{4} & \multicolumn{1}{c}{6} & \multicolumn{1}{c}{8} & \multicolumn{1}{c}{10}
    		\\ \midrule
        \multirow{3}{1ex}{\rotatebox{90}{\strut param.}}
    		&  learnable origins\footnotemark[1]       & \bf 100 + \bf 0.0 & \bf 97.2 + \bf 6.4 & \bf 79.7 + \bf 12.1 & \bf 56.7  + \bf 14.8 \\
    		& direct (origin at corner)                & 99.4 + 2.9 & 94.1 + 16.4 & 78.5 + 15.7 & 47.7 + 17.9 \\
    		& direct (fixed origin at 0)               & 99.9 + 0.6 & 87.6 + 6.4 & 65.3 + 11.9 & 46.7 + 11.5 \\
				\midrule                           
        \multirow{2}{1ex}{\rotatebox{90}{\strut basis}}
    		&  $\Delta$ basis\footnotemark[1]          & \bf 100 + \bf 0.0 & \bf 97.2 + \bf 6.4 & \bf 79.7 + \bf 12.1 & \bf 56.7  + \bf 14.8 \\
    		& canonical                                & 8.4 + 13.3 & 2.0 + 2.6 & 0.2 + 0.3 & 0.0 + 0.1 \\
        \midrule
				\multirow{5}{1ex}{\rotatebox{90}{\strut min}}     
    		&  hard                                    & 88.2 + 13.4 & 64.3 + 14.6 & 45.1 + 14.1 & 32.3 + 17.4 \\
    		& soft ($\tau=0.5$)\footnotemark[1]        & \bf 100 + \bf 0.0 & \bf 97.2 + \bf 6.4 & \bf 79.7 + \bf 12.1 & \bf 56.7  + \bf 14.8 \\
    		& soft ($\tau=1.0$)                        & 99.9 + 0.4 & 95.6 + 9.6 & 70.3 + 15.5 & 51.2 + 16.4 \\
    		& soft ($\tau=2.0$)                        & 98.8 + 3.1 & 90.6 + 14.3 & 66.4 + 12.5 & 51.2 + 9.5 \\
    		& soft ($\tau=5.0$)                        & 97.5 + 11.1 & 90.2 + 9.1 & 64.2 + 11.8 & 49.7 + 10.4 \\
    	\bottomrule
    	\end{tabular*}
    \label{sup:tab:ablations:set-covering}
		\vspace{1\baselineskip}
\end{table*}

In our cases, we ensure the feasibility of the targets by first projecting $\vy-\d\vy$ into the hypercube as
\begin{equation}\label{eq:proj-hypercube}
	(\proj{\vy'})_i = \max(\min(y_i-\d y_i,u),l)
\end{equation}
and then decomposing the projected gradient 
\begin{equation}
	\proj{\d\vy}=\vy-\proj{\vy'}
\end{equation}
into integer steps as before.

\subsection{Decomposition vs.~Ground Truth}
In the case of direct access to the ground-truth solution $\vy^*$ (\ie when \method is the last component in the architecture), we can always use $\vy^*$ as a replacement for our targets $\vy'_k$, and there is no need to decompose the incoming gradient.
In our demonstrations, we do in principle have access to $\vy^*$, but, we still employ our decomposition to demonstrate the applicability of \method also in the general case.

\subsection{Proofs}
To recover the situation from the method section, set $\vx$ as one of the inputs $\mA$, $\vb$, or $\vc$.

\begin{prop} \label{P:y-diff}
Let $\vy\colon\R^\ell\to\R^n$ be differentiable at $\vx\in\R^\ell$ and let $L\colon\R^n\to\R$ be differentiable at $\vy=\vy(\vx)\in\R^n$.
Denote $\d\vy=\partial L/\partial \vy$ at $\vy$.
Then the distance between $\vy(\vx)$ and $\vy-\d\vy$ is minimized along the direction $\partial L/\partial\vx$, where $\partial L/\partial\vx$ stands for the derivative of $L(\vy(\vx))$ at $\vx$.
\end{prop}

\begin{proof}
For $\xi\in\R^\ell$, let $\varphi(\xi)$ denote the distance between $\vy(\vx-\xi)$ and the target $\vy(\vx)-\d\vy$, \ie
\begin{equation*}
	\varphi(\xi)
		= \bigl\| \vy(\vx - \xi) - \vy(\vx) + \d\vy \bigr\|.
\end{equation*}
There is nothing to prove when $\d\vy=0$ as $\vy(x)=\vy-\d\vy$ and there is no room for any improvement.
Otherwise, $\varphi$ is positive and differentiable in a neighborhood of zero.
The Fr\'echet derivative of $\varphi$ reads as
\begin{equation*}
	\varphi'(\xi)
		= \frac{ - \bigl[\vy(\vx - \xi) - \vy(\vx) + \d\vy\bigr]
					\cdot\frac{\partial\vy}{\partial\vx}(\vx - \xi)}
				{\bigl\| \vy(\vx - \xi) - \vy(\vx) + \d\vy \bigr\|},
\end{equation*}
hence
\begin{equation}
	\varphi'(0)
		= - \frac{1}{\|\d\vy\|}
			\frac{\partial L}{\partial\vy}
				\cdot
			\frac{\partial\vy}{\partial\vx}
		= - \frac{1}{\|\d\vy\|}
			\frac{\partial L}{\partial\vx},
\end{equation}
where the last equality follows by the chain rule.
Therefore, the direction of the steepest descent coincides with the direction of the derivative $\partial L/\partial\vx$,
as $\|\d\vy\|$ is a scalar.
\end{proof}

\begin{proof}[Proof of Proposition~\ref{P:basis}]
We prove that
\begin{equation} \label{E:claim}
	\sum_{j=\ell}^n u_j \ve_{k_j}
		= \sum_{j=\ell}^n \lambda_j \Delta_j
			- |u_\ell| \sum_{j=1}^{\ell-1} \sign(u_j) \ve_{k_j}
\end{equation}
for every $\ell=1,\ldots,n$,
where we abbreviate $u_j=\d\vy_{k_j}$.
The claimed equality \eqref{E:decomposition}
then follows from~\eqref{E:claim} in the special case $\ell=1$.

We proceed by induction. In the first step we show~\eqref{E:claim} for $\ell=n$.
Definition of $\Delta_n$ \eqref{E:basis} yields
\begin{align*}
	& \lambda_n \Delta_n
			- |u_n| \sum_{j=1}^{n-1} \sign(u_j) \ve_{k_j}
			\\
	&\quad
		=	|u_n| \sum_{j=1}^{n} \sign(u_j) \ve_{k_j}
			- |u_n| \sum_{j=1}^{n-1} \sign(u_j) \ve_{k_j}
			\\
	&\quad
		= u_n \ve_{k_n}.
\end{align*}
Now, assume that \eqref{E:claim} holds for $\ell+1\ge 2$. We show
that \eqref{E:claim} holds for $\ell$ as well. Indeed,
\begin{align*}
	& \sum_{j=\ell}^n \lambda_j \Delta_j
			- |u_\ell| \sum_{j=1}^{\ell-1} \sign(u_j) \ve_{k_j}
			\\
	&\quad
		= \sum_{j=\ell+1}^n \lambda_j \Delta_j
			- |u_{\ell+1}| \sum_{j=1}^{\ell} \sign(u_j) \ve_{k_j}
			+ \lambda_\ell \Delta_\ell
			\\
	&\qquad
			+ |u_{\ell+1}| \sum_{j=1}^{\ell} \sign(u_j) \ve_{k_j}
			- |u_\ell| \sum_{j=1}^{\ell-1} \sign(u_j) \ve_{k_j}
			\\
	&\quad
		= \sum_{j={\ell+1}}^n u_j \ve_{k_j}
			+\bigl(|u_\ell| - |u_{\ell+1}|\bigr) \sum_{j=1}^{\ell} \sign(u_j) \ve_{k_j}
			\\
	&\qquad
			+ |u_{\ell+1}| \sum_{j=1}^{\ell} \sign(u_j) \ve_{k_j}
			- |u_\ell| \sum_{j=1}^{\ell-1} \sign(u_j) \ve_{k_j}
			\\
	&\quad
		= \sum_{j={\ell+1}}^n u_j \ve_{k_j}
			+ \sign(u_\ell)|u_\ell| \ve_{k_\ell}
		= \sum_{j=\ell}^n u_j \ve_{k_j},
\end{align*}
where we used the definitions of $\Delta_\ell$
and $\lambda_\ell$.
\end{proof}

\end{document}